\documentclass[10pt,journal,cspaper,compsoc]{IEEEtran}

\usepackage{times,amsmath,epsfig}
\usepackage{epstopdf}
\usepackage{graphicx}
\usepackage{url}
\usepackage{subfigure}
\usepackage{amssymb,amsfonts,verbatim}
\usepackage{amsthm}
\usepackage{bm}
\usepackage{mathtools}
\usepackage{multirow}
\usepackage{graphicx}
\usepackage{subfigure}
\usepackage{algorithm}
\usepackage{algorithmic}
\usepackage{xcolor}
\usepackage{adjustbox}[2011/08/07]
\usepackage{xspace}
\usepackage{booktabs} 

\newtheorem{theorem}{Theorem}

\begin{document}
\title{Optimized Cartesian $K$-Means}

\author{Jianfeng Wang, Jingdong Wang, Jingkuan Song, Xin-Shun Xu, Heng Tao Shen, Shipeng Li
\IEEEcompsocitemizethanks{
\IEEEcompsocthanksitem
Jianfeng Wang is with University of Science and Technology of China.\protect\\
Email: wjf2006@mail.ustc.edu.cn.
\IEEEcompsocthanksitem
Jingdong Wang and Shipeng Li are with Microsoft Research, Beijing, P.R. China.\protect\\
Emails:\{jingdw, spli\}@microsoft.com.
\IEEEcompsocthanksitem
Xin-Shun Xu is with Shandong University.\protect\\
Email: xuxinshun@sdu.edu.cn.
\IEEEcompsocthanksitem
Jingkuan Song and Heng Tao Shen are with School of Information Technology and Electrical Engineering,
The University of Queensland, Australia.\protect\\
Email:\{jk.song,shenht\}@itee.uq.edu.au.}}


\IEEEcompsoctitleabstractindextext{%
\begin{abstract}
Product quantization-based approaches are effective  to encode high-dimensional data points for approximate nearest neighbor search.
The space is decomposed into a Cartesian product of low-dimensional subspaces, each of which generates a sub codebook.
Data points are encoded as compact binary codes using these sub codebooks, and the distance between two data points can be approximated efficiently from their codes by the precomputed lookup tables.
Traditionally, to encode a subvector of a data point in a subspace, only one sub codeword in the corresponding sub codebook is selected, which may impose strict restrictions on the search accuracy.
In this paper, we propose a novel approach, named Optimized Cartesian $K$-Means (OCKM), to better encode the data points for more accurate approximate nearest neighbor search.
In OCKM, multiple sub codewords are used to encode the subvector of a data point in a subspace.
Each sub codeword stems from different sub codebooks in each subspace, which are optimally generated with regards to the minimization of the distortion errors.
The high-dimensional data point is then encoded as the concatenation of the indices of multiple sub codewords from all the subspaces.
This can provide more flexibility and lower distortion errors than traditional methods.
Experimental results on the standard real-life datasets demonstrate the superiority over state-of-the-art approaches for approximate nearest neighbor search.
\end{abstract}
}

\maketitle

\IEEEdisplaynotcompsoctitleabstractindextext

%
\IEEEpeerreviewmaketitle

\newcommand{\figSingleTwoWidth}{0.49}
\newcommand{\figDoubleThreeWidth}{0.32}

\section{Introduction}
Nearest neighbor (NN) search in large data
sets has wide applications in information retrieval, computer
vision, machine learning,
pattern recognition, recommendation system, etc.
However, exact NN search is often
intractable
because of the large scale of the database and
the curse of the high dimensionality.
Instead, approximate nearest neighbor (ANN) search is
more practical and can achieve orders of magnitude speed-ups than exact NN search with near-optimal accuracy~\cite{SDI06}.

There has been a lot of research interest on designing effective data structures,
such as $k$-d tree~\cite{FriedmanBF77},
randomized $k$-d forest~\cite{Silpa-AnanH08},
FLANN~\cite{MujaL09},
trinary-projection tree~\cite{JiaWZZH10, WangWJLZZH14},
and neighborhood graph search~\cite{AryaM93b, WangL12, WangWZGLG13a, WangWZGLG13b}.

The hashing algorithms
have been attracting a large amount of attentions
recently
as the storage cost is small
and the distance computation
is efficient.
Such approaches map data points to compact binary codes
through a hash function,
which can be generally expressed as
\begin{align*}
\mathbf{b} = \mathbf{h}(\mathbf{x}) \in \{0, 1\}^{L},
\end{align*}
where $\mathbf{x}$ is a $P$-dimensional real-valued point,
$\mathbf{h}(\cdot)$ is the hash function,
and $\mathbf{b}$ is a binary vector with $L$ entries.
For description convenience,
we will use a vector or a code
to name $\mathbf{b}$
interchangeably.

The pioneering hashing work, locality sensitive hashing (LSH)~\cite{DatarIIM04, IndykM98},
adopts random linear projections and the similarity preserving is probabilistically guaranteed.
Other approaches based on random functions include kernelized LSH~\cite{KulisG12}, non-metric LSH~\cite{MuY10}, LSH from shift-invariant kernels~\cite{RaginskyL09},
and super-bit LSH~\cite{JiLYZT12}.

To preserve some notion of similarities,
numerous efforts have been devoted to
finding a good hash function
by exploring the distribution
of the specific data set.
Typical approaches are
unsupervised hashing~\cite{GongL11, KongL12, KulisD09, StrechaBBF12, WangWYL13, WeissTF08,  XuWLZLY11, ZhuHCCS13}
and supervised hashing~\cite{liuWJ12, NorouziF11},
with kernelized version~\cite{HeLC10, LiuHLL12},
and extensions to multi-modality~\cite{SongYHSH11, song2013inter, ZhuHSZ13}, etc.
Those algorithms usually
use Hamming distance,
which is only able to produce a few distinct
distances,
resulting in limited ability and flexibility of distance
approximation.


The quantization-based algorithms have been shown
to achieve superior performances~\cite{JegouDS11, NorouziF13}.
The representative algorithms include
product quantization (PQ)~\cite{JegouDS11}
and Cartesian $K$-means (CKM)~\cite{NorouziF13},
which are modified versions of the conventional $K$-means algorithm~\cite{macqueen1967some}.
The quantization approaches typically learn a \textit{codebook} $\{\mathbf{d}_1, \cdots, \mathbf{d}_K\}$,
where each \textit{codeword} $\mathbf{d}_k$ is a $P$-dimensional vector.
The data point $\mathbf{x}$ is encoded in the following way,
\begin{align}
k^* = \arg\min\nolimits_{k \in \{1, 2, \cdots, K\}}{\|\mathbf{x} - \mathbf{d}_k\|_2^2},
\label{eqn:hash_func_non_decomposed}
\end{align}
where $\|\cdot\|_2$ denotes the $l_2$ norm. 
The index $k^*$ indicates which codeword is the closest to $\mathbf{x}$
and can be represented as a binary code of length $\lceil\log_2(K)\rceil$\footnote{In the following, we omit the $\lceil \cdot \rceil$
operator without affecting the understanding.}.

The crucial problem for quantization algorithms
is how to learn the codebook.
In the traditional $K$-means,
the codebook is composed of the cluster centers
with a minimal squared distortion error.
The drawbacks when applying $K$-means to ANN search
include that the size of the codebook is quite limited
and computing the distances between the query and the codewords
is expensive.
PQ~\cite{JegouDS11}
addresses this problem
by splitting the $P$-dimensional space into multiple disjoint subspaces
and making the codebook as the Cartesian product of the \textit{sub codebooks},
each of which is learned on each subspace using the conventional $K$-means algorithm.
The compact code is formed by concatenating
the indices of the selected sub codeword within each sub codebook.
CKM~\cite{NorouziF13} improves PQ
by optimally rotating the $P$ dimensional space to give a lower distortion error.



In PQ and
CKM,
only one sub codeword on each subvector is used to quantize the data points.
which results in limited capability of reducing the distortion error
and thus limited search accuracy.
In this paper,
we first present a simple algorithm,
extended Cartesian $K$-means (ECKM), which extends CKM
by using multiple (e.g., $C$) sub codewords for a data point from the sub codebook in each subspace.
Then, we propose
the optimized Cartesian $K$-means (OCKM) algorithm,
which learns $C$ sub codebooks in each subspace
instead of a single sub codebook like ECKM,
and selects $C$ sub codewords, each chosen from a different sub codebook.
We show that both  PQ and CKM 
are constrained versions of our OCKM 
under the same code length,
which suggests that
our OCKM can
lead to a lower quantization error
and thus a higher search accuracy.
Experimental results
also validate that
our OCKM
achieves superior performance.

The remainder of this paper is organized as follows.
Related work is first reviewed in Sec.~\ref{sec:preliminary}.
The proposed ECKM is introduced in Sec.~\ref{sec:sck_means}, followed by the OCKM in Sec.~\ref{sec:ock_means}.
Discussions and experimental results are given in Sec.~\ref{sec:discussion} and~\ref{sec:exp}, respectively.
Finally, a conclusion is made in Sec.~\ref{sec:conclusion}.

\section{Related work}\label{sec:preliminary}
Hashing is an emerging technique to represent the high-dimensional vectors as binary
codes for ANN search, and has
achieved a lot of success in multimedia applications,
e.g. image search~\cite{HeFLCLCC12, KuoCCH09}, video retrieval~\cite{CaoLMC12, SongYHSH11},
event detection~\cite{RevaudDSJ13}, document retrieval~\cite{SalakhutdinovH09}.

According to the form of the
hash function,
we roughly categorize the binary
encoding approaches as those based on  Hamming embedding
and  on quantization.
Roughly, the former adopts the Hamming distance 
as the dissimilarity between the codes, while the latter 
does not.

Table~\ref{tbl:notations} illustrates part of the
notations and descriptions used in the paper. Generally, we
use the uppercase unbolded symbol as
a constant,
the lowercase unbolded
as the index,
the uppercase bolded as the matrix
and
the lowercase bolded as the vector.

\begin{table}
\centering
\caption{Notations and descriptions.}
\label{tbl:notations}
\begin{tabular}{cc}
\toprule
Symbol & Description \\
\midrule
$N$ & number of training points \\
$P$ & dimension of training points \\
$M$ & number of subvectors \\
$S$ & number of dimensions on each subvector \\
$K$ & number of (sub) codewords \\
$m$ & index of the subvector \\
$i$ & index of the training point \\
$\mathbf{R}$ & rotation matrix \\
$\mathbf{D}^{m}$ & codebook on $m$-th subvector\\
$\mathbf{b}^{m}_{i}$ & $1$-of-$K$ encoding vector on $m$-th subvector\\
\bottomrule
\end{tabular}
\end{table}

\subsection{Hamming embedding}
Linear mapping is one of typical hash functions.
Each bit is calculated by
\begin{align}
h_i(\mathbf{x}) =
\operatorname{sign}(\mathbf{w}_i^{T}\mathbf{x} + u_i),
\end{align}
where $\mathbf{w}_i$ is the projection vector, $u_i$ is the offset,
and
$\operatorname{sign}(z)$ is a sign
function which is $1$ if $z>0$,
and $0$
otherwise.

Such approaches include
~\cite{DatarIIM04, GongL11, KongL12}.
The differences mainly reside in how to
obtain the parameters in the hash function.
For example, LSH~\cite{DatarIIM04}
adopts a random parameter and the
similarity is probability preserved.
Iterative quantization
hashing~\cite{GongL11}
constructs hash functions by rotating
the axes so that the difference between the binary codes
and the projected data is minimized.

Another widely-used approach
is the kernel-based hash
function~\cite{
HeLC10, KulisD09, KulisG12, 
LiuHLL12},
i.e.
\begin{align}
h_i(\mathbf{x}) =
\operatorname{sign}(\sum_{j}w_{ij}{
\kappa(\mathbf{x}, \mathbf{z}_j)}),
\end{align}
where $\mathbf{z}_j$ is the vector in
the same space with $\mathbf{x}$,
and $\kappa(\cdot, \cdot)$ is the kernel function.
The cosine function can also be used to generate the binary codes, such as 
in~\cite{WeissTF08}.

\subsection{Quantization}

In the quantization-based encoding methods, different constraints on the codeword 
lead to different approaches, i.e.
$K$-Means~\cite{Lloyd82, macqueen1967some},
Product Quantization (PQ)~\cite{JegouDS11} and
Cartesian $K$-Means (CKM)~\cite{NorouziF13}.

\subsubsection{$K$-Means}

Given $N$ $P$-dimensional points $\mathcal{X} = \{\mathbf{x}_1, \cdots, \mathbf{x}_N\}\subset \mathbb{R}^{P}$, the $K$-means algorithm partitions the database
into $K$ clusters, each of which associates
one codeword $\mathbf{d}_i\in
\mathbb{R}^{P}$. Let $\mathbf{D} =
[\mathbf{d}_1, \cdots, \mathbf{d}_K]
\subset
\mathbb{R}^{P}$ be the
corresponding codebook.
Then the codebook is learned by minimizing the within-cluster distortion, i.e.
\begin{align*}
\min~~&{\sum_{i = 1}^{N}{\|\mathbf{x}_i - \mathbf{D}\mathbf{b}_i\|_2^2}} \\
\operatorname{s.t.}~~&\mathbf{b}_i \in \{0, 1\}^{K} \\
& \|\mathbf{b}_i\|_1 = 1 ~~ i\in \{1, \cdots, N\} 
\end{align*}
where $\mathbf{b}_i$ is a $1$-of-$K$ encoding vector ($K$ dimensions with one $1$ and $K - 1$ $0$s. )
to indicate which codeword is used to quantize $\mathbf{x}_i$,
and $\|\cdot\|_1$ is the $l_1$ norm.

The problem can be solved 
by iteratively alternating
optimization with respect to $\mathbf{D}$ and  $\{\mathbf{b}_i\}_{i = 1}^{N}$~\cite{Lloyd82}.

\subsubsection{Product Quantization}
One issue of $K$-Means is the size of the codebook is quite limited due to 
the storage and computational cost.
To address the problem, PQ~\cite{JegouDS11}
splits each $\mathbf{x}_i$ into $M$
disjoint subvectors.
Assume the $m$-th subvector contains $S_m$ dimensions and then $\sum_{m = 1}^{M}{S_m} = P$.
Without loss of generality, $S_m$ is set to $S \triangleq P / M$ and $P$ is assumed to be divisible by $M$.
On the $m$-th subvector, $K$-means is performed to obtain $K$ \textit{sub codewords}. By this method, it generates $K^M$ clusters with only $O(KP)$ storage, while $K$-means requires $O(K^MP)$ storage
with the same number of clusters.
Meanwhile, the computing complexity is reduced from $O(K^MP)$ to $O(KP)$ to encode one data point.

Let $\mathbf{D}^m \in \mathbb{R}^{S \times K}$ be the matrix of the $m$-th sub codebook and each column is a $S$-dimensional sub codeword.
PQ can be taken as optimizing the following problem with respect to $\{\mathbf{D}^m\}_{m = 1}^{M}$ and $\{\mathbf{b}_i^m\}_{i = 1, m = 1}^{N, M}$.
\begin{align}
\begin{split}
\min ~~ &f_{\text{pq}, M, K} = \sum_{i = 1}^{N}
{
	\left\|
		\mathbf{x}_i -
			\begin{bmatrix}
			\mathbf{D}^1 \mathbf{b}_i^1\\
			\vdots \\
			\mathbf{D}^{M} \mathbf{b}_i^M
		\end{bmatrix}
	\right\|_2^2
} \\
\operatorname{s.t.}~~&\mathbf{b}_i^{m} \in \{0, 1\}^{K}  \\
& \|\mathbf{b}_i^{m}\|_1 = 1~~ i \in \{1, \cdots, N\},  m\in \{1, \cdots, M\}
\end{split}
\label{eqn:pq}
\end{align}
where
$\mathbf{b}_i^{m}$ is
also the $1$-of-$K$ encoding vector
on the $m$-th subvector and the index of $1$ indicates which sub codeword is used to encode $\mathbf{x}_i$.

\subsubsection{Cartesian $K$-Means}
CKM~\cite{NorouziF13}
optimally rotates the original space and formulates the problem as
\begin{align}
\begin{split}
\min ~~ & f_{\text{ck}, M, K} = \sum_{i = 1}^{N}
\left\|
	\mathbf{x}_i -
	\mathbf{R}
		\begin{bmatrix}
			\mathbf{D}^1 \mathbf{b}_i^1 \\
			\vdots \\
			\mathbf{D}^{M} \mathbf{b}_i^M
		\end{bmatrix}
\right\|_2^2 \\
\operatorname{s.t.} ~~ & \mathbf{R}^T \mathbf{R} = \mathbf{I}\\
& \mathbf{b}_i^{m} \in \{0, 1\}^{K}  \\
& \|\mathbf{b}_i^{m}\|_1 = 1~~ i \in \{1, \cdots, N\}, m\in \{1, \cdots, M\}
\end{split}
\label{eqn:ck_means}
\end{align}
The rotation matrix $\mathbf{R}$ is optimally learned by minimizing the distortion.


If $\mathbf{R}$ is constrained to
be the identity matrix $\mathbf{I}$,
it will be reduced to Eqn.~\ref{eqn:pq}.
Thus, we can assert that under the optimal solutions, we have $f^*_{\text{ck}, M, K} \le f^*_{\text{pq}, M, K}$,
where the asterisk superscript
indicates the objective function with the optimal parameters.

\section{Extended Cartesian $K$-Means}
\label{sec:sck_means}

In both PQ and CKM, only one sub codeword is used to encode the subvector.
To make the representation more flexible, we propose 
the extended Cartesian $K$-means (ECKM), where multiple sub codewords can be used in each subspace.

Mathematically,
we allow the $l_1$ norm of $\mathbf{b}_i^m$ to be a pre-set number $C$ ($C\ge 1$), instead of limiting it to be exactly 1. Meanwhile, any entry of $\mathbf{b}_i^m$ is relaxed as a non-negative integer instead of a binary value. The formulation is
\begin{align}
\begin{split}
\min ~~ & f_{\text{eck}, M, K, C} = \sum_{i = 1}^{N}
\left\|
	\mathbf{x}_i -
		\mathbf{R}
		\begin{bmatrix}
			\mathbf{D}^1 \mathbf{b}_i^1\\
			\vdots \\
			\mathbf{D}^{M} \mathbf{b}_i^M
		\end{bmatrix}
\right\|_2^2\\
\operatorname{s.t.} ~~ & \mathbf{R}^T \mathbf{R} = \mathbf{I}\\
& \mathbf{b}_i^m \in \mathbb{Z}_{+}^{K}\\
& \|\mathbf{b}_i^m\|_1 = C
\end{split}
\label{eqn:ick}
\end{align}
where $\mathbb{Z}_{+}$ denotes the set of non-negative integers.
The constraint is applied on all the points $i \in \{1, \cdots, N\}$
and on all the subspaces $m\in \{1, \cdots, M\}$. 
In the following, we omit the range of $i, m$
without confusion.

For the $m$-th sub codebook $\mathbf{D}^m \in \mathbb{R}^{S\times K}$, traditionally only one sub codeword can be selected
and there are only $K$ choices to encode the $m$-th subvector of
$\mathbf{R}^{T}\mathbf{x}_i$.
In the extended version,
any feasible $\mathbf{b}_i^m$
satisfying
$\mathbf{b}_i^m \in \mathbb{Z}_{+}^{K}$
and $\|\mathbf{b}_i^m\|_1 = C$
constructs a
quantizer, i.e. $\mathbf{D}^{m}\mathbf{b}_i^m$.
Thus, the total number of choices is $\binom{K + C - 1}{K - 1} \ge K$.
For example with $K = 256$ and $C = 2$, the difference is $\binom{K + C - 1}{K - 1} = 32896 \gg K = 256$.
With a more powerful representation, the distortion errors can be potentially reduced.

In theory, $\log_2\binom{K + C - 1}{K - 1}$ bits can be used to encode one $\mathbf{b}_i^m$, and the code length is $M\log_2(\binom{K + C - 1}{K - 1})$.
Practically, we use $\log_2(K)$ bits to encode one position of $1$. The $l_1$ norm of $\mathbf{b}_i^m$
is $C$, which can be interpreted that there are $C$ $1$s in $\mathbf{b}_i^m$.
Then $MC\log_2(K)$ bits are allocated to encode one data point.

\subsection{Learning}

Similar to~\cite{NorouziF13},
we present an iterative coordinate
descent algorithm to solve the problem in Eqn.~\ref{eqn:ick}.
There are three kinds of unknown variables, $\mathbf{R}$, $\mathbf{D}^m$, and $\mathbf{b}_i^m$.
In each iteration, two of them are fixed, and the other one is optimized.

\subsubsection{Solve $\mathbf{R}$ with $\mathbf{b}_i^m$ and $\mathbf{D}^m$ fixed}

With
\begin{align*}
\mathbf{X} &\triangleq
		\begin{bmatrix}
		\mathbf{x}_1 & \cdots & \mathbf{x}_N
		\end{bmatrix} \\
\mathbf{D} &\triangleq
		\begin{bmatrix}
		\mathbf{D}^{1} & & \\
						& \ddots & \\
						& 			& \mathbf{D}^{M}
		\end{bmatrix} \\
\mathbf{B} &\triangleq
	\begin{bmatrix}
	\mathbf{b}_1 & \cdots & \mathbf{b}_N
	\end{bmatrix} \\
\mathbf{b}_i &\triangleq
		\begin{bmatrix}
		{\mathbf{b}_i^1}^{T} & \cdots  &	{\mathbf{b}_i^M}^{T}
		\end{bmatrix}^{T},
\end{align*}
we re-write the objective function of Eqn.~\ref{eqn:ick} in a matrix form as
\begin{align*}
\|\mathbf{X} - \mathbf{R}\mathbf{D}\mathbf{B}\|_F^2,
\end{align*}
where $\|\cdot\|_F$ is the Frobenius norm.
The problem of solving $\mathbf{R}$ is the classic Orthogonal Procrustes problem~\cite{Peter66} and the solution can be obtained as follows: if SVD of $\mathbf{X}{(\mathbf{D}\mathbf{B})}^{T}$ is $\mathbf{X}{(\mathbf{D}\mathbf{B})}^{T} = \mathbf{U}\Sigma \mathbf{V}^T$, the optimal $\mathbf{R}$ will be $\mathbf{U}\mathbf{V}^{T}$.

\subsubsection{Solve $\mathbf{D}^m$ with $\mathbf{b}_i^m$ and $\mathbf{R}$ fixed}

Let $\mathbf{z}_i \triangleq \mathbf{R}^{T}\mathbf{x}_i$
and the $m$-th subvector of $\mathbf{z}_i$ be $\mathbf{z}_i^m$.
The objective function of Eqn.~\ref{eqn:ick} can also be written as,
\begin{align}
\sum_{i = 1}^{N} \sum_{m = 1}^{M}{\|\mathbf{z}_i^m - \mathbf{D}^m\mathbf{b}_i^m\|_2^2} = \sum_{m = 1}^{M}\|\mathbf{Z}^{m} - \mathbf{D}^{m}\mathbf{B}^{m}\|_F^2,
\label{eqn:represent_z_i_m}
\end{align}
where
\begin{align*}
\mathbf{Z}^m & \triangleq [\mathbf{z}_1^m, \cdots, \mathbf{z}_N^m] \\
\mathbf{B}^m & \triangleq [\mathbf{b}_1^m, \cdots, \mathbf{b}_N^m].
\end{align*}

Each $\mathbf{D}^{m}$ can be individually optimized  as
$(\mathbf{Z}^{m} {\mathbf{B}^{m}}^{T}) (\mathbf{B}^{m} {\mathbf{B}^{m}}^{T})^{+}$,
where $(\cdot)^{+}$ denotes the matrix (pseudo)inverse.

\subsubsection{Solve $\mathbf{b}_i^m$ with $\mathbf{D}^m$ and $\mathbf{R}$ fixed}

From Eqn.~\ref{eqn:ick} and Eqn.~\ref{eqn:represent_z_i_m}, $\mathbf{b}_i^m$ can be solved by optimizing
\begin{align*}
\min ~~&g_{\text{eck}} (\mathbf{b}_i^m) = \|\mathbf{z}_i^m - \mathbf{D}^m\mathbf{b}_i^m\|_2^2 \\
\operatorname{s.t.}~~& \mathbf{b}_i^m \in \mathbb{Z}_{+}^{K} \\
&\|\mathbf{b}_i^m\|_1 = C
\end{align*}

This is an integer quadratic programming and challenging to solve.
Here, we present a simple but practically efficient algorithm, based on matching pursuit~\cite{MallatZ93} and illustrated in Alg.~\ref{alg:ick_code}.
In each iteration, we hold a residual variable $\mathbf{r}$, initialized by $\mathbf{z}_i^m$ (Line~\ref{line:init_r} in Alg.~\ref{alg:ick_code}).
Let $\mathbf{d}_{k}^{m}$ be the $k$-th column of $\mathbf{D}^{m}$.
Each column is scanned to obtain the best one to minimize the distortion error (Line~\ref{line:best_k}), i.e.
\begin{align*}
k^* = \arg\min_{k}{\|\mathbf{r} - \mathbf{d}_{k}^{m}\|_2^2}.
\end{align*}
Then $\mathbf{r}$ is subtracted by $\mathbf{d}_{k^*}^m$ (Line~\ref{line:update_r}) for the next iteration, and the $k^*$-th dimension of $\mathbf{b}_i^m$ increases by $1$ (Line~\ref{line:update_b}) to indicate the $k^*$-th sub codeword is selected.
The process stops until $C$ iterations are reached.

\begin{algorithm}[t]
\caption{Code Generation for ECKM}
\label{alg:ick_code}
\begin{algorithmic}[1]
    \REQUIRE
        $\mathbf{z}_i^m$, $\mathbf{D}^m \in \mathbb{R}^{S\times K}$, $C$
    \ENSURE
        $\mathbf{b}_i^m$
    \STATE
    	$\mathbf{b}_i^m = \text{zeros}(K, 1)$
    \STATE
         $\mathbf{r} = \mathbf{z}_i^m$
         \label{line:init_r}
    \FOR{$c = 1 : C$}
    	\STATE
        	$k^* = \arg\min_{k}{\|\mathbf{r} - \mathbf{d}_{k}^m\|_2^2}$
        	\label{line:best_k}
        \STATE
        	$\mathbf{r} = \mathbf{r} - \mathbf{d}_{k^*}^{m}$
        	\label{line:update_r}
		\STATE
        	${b}_i^m(k^*) = {b}_i^m(k^*) + 1$
        	\label{line:update_b}
    \ENDFOR
\end{algorithmic}
\end{algorithm}

\section{Optimized Cartesian $K$-Means} \label{sec:ock_means}
Before introducing the proposed OCKM, we first present another equivalent formulation of the ECKM.
Since each entry of $\mathbf{b}_i^m$ in Eqn.~\ref{eqn:ick} is a non-negative integer, and the sum of all the entries is $C$,
we replace it by
\begin{align}
\mathbf{b}_i^m = \sum_{c = 1}^{C}{\mathbf{b}_i^{m, c}}
\label{eqn:construct_b_im}
\end{align}
with
\begin{align}
\begin{split}
& \mathbf{b}_{i}^{m, c}  \in \{0, 1\}^{K} \\
& \|\mathbf{b}_i^{m, c}\|_1  = 1.
\end{split}
\label{eqn:seperate_new_constraint}
\end{align}

Given any feasible $\mathbf{b}_i^{m}$,
we can always find at least one group of $\{\mathbf{b}_i^{m, c}\}_{c = 1}^{C}$ satisfying  Eqn.~\ref{eqn:seperate_new_constraint} and Eqn.~\ref{eqn:construct_b_im}.
Any group of $\{\mathbf{b}_i^{m, c}\}_{c = 1}^{C}$ satisfying Eqn.~\ref{eqn:seperate_new_constraint} can also construct a valid $\mathbf{b}_i^m$ by Eqn.~\ref{eqn:construct_b_im} for Eqn.~\ref{eqn:ick}.
For example, if $\mathbf{b}_i^m = \begin{bmatrix}
2 & 0 & 1 & 0
\end{bmatrix}$, we can replace it by the summation of $\begin{bmatrix}
1 & 0 & 0 & 0
\end{bmatrix}$,
$\begin{bmatrix}
1 & 0 & 0 & 0
\end{bmatrix}
$
and
$
\begin{bmatrix}
0 & 0 & 1 & 0
\end{bmatrix}
$.

Substituting Eqn.~\ref{eqn:construct_b_im} into the objective function of Eqn.~\ref{eqn:ick}, we have
\begin{align*}
f_{\text{eck}, M, K, C}
 = \sum_{i = 1}^{N}{\left\|\mathbf{x}_i - \mathbf{R}
\begin{bmatrix}
\sum_{c}{\mathbf{D}^{1}\mathbf{b}_i^{1, c}} \\
\vdots\\
\sum_{c}{\mathbf{D}^{M}\mathbf{b}_i^{M, c}}
\end{bmatrix}
\right\|_2^2}.
\end{align*}

On the $m$-th subvector, $\mathbf{b}_i^{m, c}$ represents the selected sub codeword.
There are in total of $C$ selections from a single sub codebook.
To further reduce the distortion errors, we propose to expand one sub codebook to
$C$ different sub codebooks $\mathbf{D}^{m, c}\in \mathbb{R}^{S \times K}, c\in \{1, \cdots, C\}$, each of which is used for sub codeword selection.
In summary, the formulation is as follows.
\begin{align}
\begin{split}
\min  & f_{\text{ock}, M, K, C}
 = \sum_{i = 1}^{N}{\left\|\mathbf{x}_i - \mathbf{R}
\begin{bmatrix}
\sum_{c}{\mathbf{D}^{1, c}\mathbf{b}_i^{1, c}} \\
\vdots\\
\sum_{c}{\mathbf{D}^{M, c}\mathbf{b}_i^{M, c}}
\end{bmatrix}
\right\|_2^2} \\
\operatorname{s.t.}
& ~\mathbf{R}^T \mathbf{R} = \mathbf{I}\\
& ~\mathbf{b}_i^{m, c} \in \{0, 1\}^{K} \\
& ~\|\mathbf{b}_i^{m, c}\|_1 = 1
\end{split}
\label{eqn:ock_means}
\end{align}
which we call Optimized Cartesian $K$-Means (OCKM).

Since any $\mathbf{b}_i^{m, c}$ requires $\log_2(K)$
bits to encode, the code length of representing each point is $MC\log_2(K)$.

\subsection{Learning}

Similar with ECKM, an iterative coordinate descent algorithm is employed to optimize $\mathbf{R}$, $\mathbf{D}^{m, c}$ and $\mathbf{b}_i^{m, c}$.

\subsubsection{Solve $\mathbf{R}$ with $\mathbf{D}^{m, c}$ and $\mathbf{b}_i^{m, c}$ fixed}

The objective function is re-written in a matrix form as
\begin{align*}
\|\mathbf{X} - \mathbf{R}\hat{\mathbf{D}}\hat{\mathbf{B}}\|_F^2,
\end{align*}
where
\begin{align}
\hat{\mathbf{D}} & \triangleq
\begin{bmatrix}
\hat{\mathbf{D}}^m & & \\
					&	\ddots & \\
					&			&	\hat{\mathbf{D}}^m
\end{bmatrix} 
 \\
\hat{\mathbf{D}}^m &\triangleq
\begin{bmatrix}
\mathbf{D}^{m, 1} & \cdots & \mathbf{D}^{m, C}
\end{bmatrix}  \label{eqn:hat_d_m}
\\
\hat{\mathbf{B}} & \triangleq \begin{bmatrix}
{{}\hat{\mathbf{B}}^{1}}^{T} &
\cdots &
{{}\hat{\mathbf{B}}^{M}}^{T}
\end{bmatrix}^{T} 
\\
\hat{\mathbf{B}}^{m} & \triangleq \begin{bmatrix}
\hat{\mathbf{b}}_1^m & \cdots & \hat{\mathbf{b}}_N^m
\end{bmatrix} \\
\hat{\mathbf{b}}_i^m & \triangleq 
\begin{bmatrix}
{\mathbf{b}_i^{m, 1}}^{T} &
\cdots &
{\mathbf{b}_i^{m, C}}^{T}
\end{bmatrix}^{T}.  \label{eqn:hat_b_i_m}
\end{align}
Then optimizing $\mathbf{R}$ is the Orthogonal Procrustes Problem~\cite{Peter66}.

\subsubsection{Solve $\mathbf{D}^{m, c}$ with $\mathbf{R}$ and $\mathbf{b}_i^{m, c}$ fixed}

Similar with Eqn.~\ref{eqn:represent_z_i_m} in ECKM, the objective function of OCKM can be written as
\begin{align*}
\sum_{m = 1}^{M}{\|\mathbf{Z}^m - \hat{\mathbf{D}}^{m}\hat{\mathbf{B}}^m\|_F^2}.
\end{align*}
Each $\hat{\mathbf{D}}^{m}$ can also be individually solved  by the matrix (pseudo)inversion.

\subsubsection{Solve $\mathbf{b}_i^{m, c}$ with $\mathbf{R}$ and $\mathbf{D}^{m, c}$ fixed}\label{subsubsec:solve_b}

The sub problem is
\begin{align*}
\min ~~&g_{\text{ock}}(\hat{\mathbf{b}}_i^{m, c}) = \|\mathbf{z}_i^m - \sum_{c = 1}^{C}{\mathbf{D}^{m, c}\mathbf{b}_i^{m, c}}\|_2^2 \\
\operatorname{s.t.}~~& \mathbf{b}_i^{m, c} \in \{0, 1\}^{K} \\
&\|\mathbf{b}_i^{m, c}\|_1 = 1
\end{align*}

One straightforward method to solve the sub problem is to greedily find the best sub codeword in $\mathbf{D}^{m, c}$ one by one similar with Alg.~\ref{alg:ick_code} for ECKM.
One drawback is the succeeding sub codewords can only be combined with the previous one sub codeword.

\begin{figure}[ttt!]
 \begin{minipage}[t]{3.15in}
 \begin{algorithm}[H]
 \caption{Code generation for OCKM}
 \label{alg:opt_ck_b}
 \begin{algorithmic}[1]
     \REQUIRE
         $\mathbf{z}_i^m$, $\hat{\mathbf{D}}^{m} \in \mathbb{R}^{S\times KC}$
     \ENSURE
         $\hat{\mathbf{b}}_i^{m}$
     \STATE
     	[$\hat{\mathbf{b}}_i^{m}$, $\text{error}$] = GenCodeOck($\mathbf{z}_i^m$, $\hat{\mathbf{D}}^{m}$, 1)
 \end{algorithmic}
 \end{algorithm}
 \end{minipage}
 \begin{minipage}[t]{3.15in}
\begin{algorithm}[H]
\caption{[$\hat{\mathbf{b}}$, $\text{error}$] = GenCodeOck($\mathbf{z}_i^m$, $\hat{\mathbf{D}}^m$, $\text{idx}$)}
\label{alg:opt_ck_b_recur}
\begin{algorithmic}[1]
    \IF {$\text{idx} == \text{C}$}
	    \STATE
	    	$k^* = \arg\min_{k}{\|\mathbf{z}  - \mathbf{d}_{k}^{m, \text{idx}}\|_2^2}$
	    	\label{line:last_best}
	    \STATE
	    	$\mathbf{b} = \text{zeros}(K, 1)$
	    \STATE
	    	${b}(k^*) = 1$
	    \STATE
	    	$\text{error} = \|\mathbf{z} - \mathbf{d}^{m, \text{idx}}_{k^*}\|_2^2$
	\ELSE
	    \STATE
	    	$[k^*_1, \cdots, k^*_{T}] = \arg\min_{k}{\|\mathbf{z}_i^m - \mathbf{d}_{k}^{m, \text{idx}}\|_2^2}$ \label{line:best_T_columns}
	    \STATE
	    	$\text{best}.\text{error} = \text{LARGE}$ \label{line:init_best_error}
	    \FOR {$i = 1 : T$}
	    \STATE
	    	$k \leftarrow k_i^*$
	    \STATE
	    	$\mathbf{z}' = \mathbf{z}_i^m - \mathbf{d}_{k}^{m, \text{idx}}$ \label{line:substract_target}
	    \STATE
	    	[$\hat{\mathbf{b}}'$, $\text{error}'$] = GenCodeOck($\mathbf{z}'$, $\hat{\mathbf{D}}^{m}$, $\text{idx} + 1$) \label{line:recursive_calling}
	    \IF {$\text{error}' < \text{best}.\text{error}$}
	    \STATE
	    	$\text{best}.\text{error}= \text{error}'$
	    \STATE
	    	$\text{best}.\text{idx}= k$
	    \STATE
	    	$\text{best}.\hat{\mathbf{b}} = \hat{\mathbf{b}}'$
	    \ENDIF
	    \ENDFOR
	    \STATE
	    	$\mathbf{b}^1 = \text{zero}(K, 1)$ \label{line:init_b1}
	    \STATE
	    	${b}^1(\text{best}.\text{idx}) = 1$ \label{line:set_b1}
	    \STATE
	    $\hat{\mathbf{b}} = [\mathbf{b}^{1}; \text{best}.\hat{\mathbf{b}}]$ \label{line:construct_hat_b}
    \ENDIF
\end{algorithmic}
\end{algorithm}
 \end{minipage}
\end{figure}

To increase the accuracy with a reasonable time cost, we improve it as multiple best candidates matching pursuit. The algorithm is illustrated in Alg.~\ref{alg:opt_ck_b} and Alg.~\ref{alg:opt_ck_b_recur}.
The input is the target vector $\mathbf{z}_i^m$, and the sub codebooks $\hat{\mathbf{D}}^{m}$ (defined in Eqn.~\ref{eqn:hat_d_m}).
The output is the binary code represented as $\hat{\mathbf{b}}_i^m$ (defined in Eqn.~\ref{eqn:hat_b_i_m}).

The function $[\hat{\mathbf{b}}, \text{error}] = \text{GenCodeOck}(\mathbf{z}_i^m, \hat{\mathbf{D}}^m, \text{idx})$
in Alg.~\ref{alg:opt_ck_b_recur}
encodes $\mathbf{z}_i^m$
with the last $(C - \text{idx} + 1)$ sub codebooks $\{\mathbf{D}^{m, c}, c\in \{\text{idx}, \cdots, C\}\}$.
The encoding vector $\hat{\mathbf{b}}$ with $(C - \text{idx} + 1)K$ dimensions and  the distortion $\text{error}$ are returned.

At first, $\text{idx} = 1$ and we search the top-$T$ best columns in $\mathbf{D}^{m, \text{idx}}$ (Line~\ref{line:best_T_columns} in Alg.~\ref{alg:opt_ck_b_recur}) with $T$ being a pre-defined parameter. 
Let $\mathbf{d}_{k}^{m, \text{idx}}$ be the $k$-th column of $\mathbf{D}^{m, \text{idx}}$.
The final selected one is taken among the $T$ best candidates.
For each candidate, the target vector is substracted by the corresponding sub codeword (Line~\ref{line:substract_target}), and then the rest codes $\hat{\mathbf{b}}'$ are generated by recursively calling the function $\text{GenCodeOck}$ with the parameter $\text{idx} + 1$
(Line~\ref{line:recursive_calling}).

Among the $T$ candidates, the one with the smallest distortion error stored in $\text{best}.\text{idx}$ is selected to construct the final binary representation (Line~\ref{line:init_b1},~\ref{line:set_b1},~\ref{line:construct_hat_b}). In Line~\ref{line:init_best_error}, the error is initialized as a large enough constant $\text{LARGE}$.

\noindent\textbf{Analysis.}
The parameter $T$ controls the time cost and the accuracy towards the optimality. If the time complexity is $J(C)$, we can derive the recursive relation
\begin{align*}
J(C) = SK + TJ(C - 1).
\end{align*}
As shown in Line~\ref{line:best_T_columns} of Alg.~\ref{alg:opt_ck_b_recur},
$T$ sub codewords are selected and here we simply compare with each sub codeword, resulting in $O(SK)$ complexity. Since $T$ is generally far smaller than $K$, the cost of partially sorting to obtain the $T$ best ones can be ignored. For each of the $T$ best sub codeword, the complexity of finding the binary code in the rest sub codebooks is $J(C - 1)$ (Line~\ref{line:recursive_calling}). With $J(1) = SK$, we can derive the complexity is
\begin{align}
J(C) = SK\frac{T^C - 1}{T - 1}.
\end{align}
Since there are $M$ subvectors, the complexity of encoding one full vector is $J(C)M = PK(T^{C} - 1) / (T - 1) = O(PKT^{C - 1})$. The time cost increases with a larger $T$.

Generally, Alg.~\ref{alg:opt_ck_b} can achieve a better solution with a larger $T$.
If the position of $1$ in $\mathbf{b}_i^{m, c}$ is uniformly distributed and independent with the others, we can calculate the probability of obtaining the optimal solution by Alg.~\ref{alg:opt_ck_b}.
On each subvector, there are $K^{C}$ different cases for $\hat{\mathbf{b}}_i^{m}$.
In Alg.~\ref{alg:opt_ck_b},
Line~\ref{line:best_T_columns} is executed $C - 1$ times, and thus
$T$ sub codewords are selected for each of the first $C - 1$ sub codebooks.
All the sub codewords in the last sub codebook can be taken to be  tried to find the one with the minimal distortion (Line~\ref{line:last_best}).
Then, $T^{C - 1}K$ different cases are checked, and the probability to find the optimal solution is
\begin{align}
\frac{T^{C - 1} K}{K^{C}} = \left(\frac{T}{K}\right) ^ {C - 1}.
\end{align}

If $T = K$, the probability will be $1$. It is certain that the optimal solution can be found, but with a high time cost. The probability increases with a larger $T$. Meanwhile, it decreases exponentially with $C$. Generally, we set $C = 2$ to have a better sub optimal solution.
Fig.~\ref{fig:sift1m3_diff} illustrates the relationship between the optimized distortion errors and $T$ on the SIFT1M training set, which is described in Sec.~\ref{sec:exp}.
In practice, we choose $T = 10$ as a tradeoff.

\begin{figure}
\centering
\includegraphics[width = 0.85\linewidth]{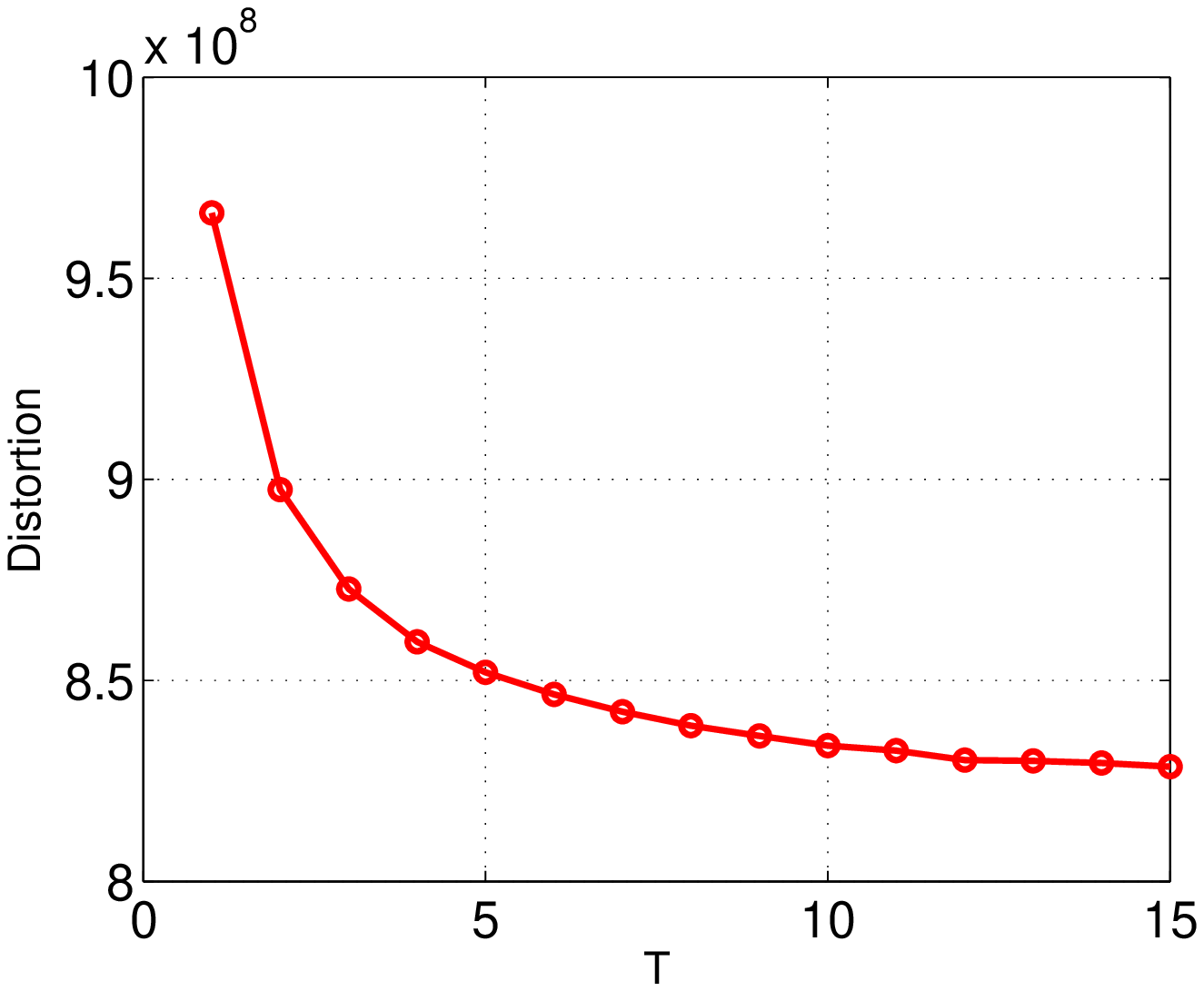}
\caption{Distortion errors on the training set of SIFT1M vs different $T$(s) with $M = 8$, $K = 256$ and $C = 2$.}
\label{fig:sift1m3_diff}
\end{figure}

\section{Discussions}\label{sec:discussion}

\subsection{Connections}
\label{sec:connection}

Our approaches are closely related with PQ~\cite{JegouDS11}
and CKM~\cite{NorouziF13}.
PQ splits the original vector into multiple subvectors to
address the scalability issues.
CKM rotates the space optimally and thus
can achieve better accuracy.
In each subspace, both PQ and CKM generate a single sub codebook and
choose one
sub codeword to quantize the original point.
Our ECKM extends the idea by choosing multiple sub codewords from the
single sub codebook, while our OCKM generates multiple sub codebooks,
each of which contributes one sub codeword.

Next, we theoretically discuss the relations between our OCKM
and others.

\begin{theorem}
\label{thm:same_number_subvector}
Under optimal solutions, we have:
\begin{align}
f_{\text{ock}, M, K, C}^*
& \le f_{\text{ck}, M, K}^*
\label{eqn:less_equal_ck_ock}
\\
f_{\text{ock}, M, K, C}^* & \le
f_{\text{eck}, M, K, C}^{*}.
\label{eqn:less_equal_ours}
\end{align}
\end{theorem}

\begin{proof}
If we limit $\mathbf{D}^{m, c_1} = \mathbf{D}^{m, c_2}, c_1, c_2 \in \{1, \cdots, C\}$ in
Eqn.~\ref{eqn:ock_means},
OCKM is reduced to the ECKM in Eqn.~\ref{eqn:ick} by relations in Eqn.~\ref{eqn:construct_b_im} and Eqn.~\ref{eqn:seperate_new_constraint}, which proves the Eqn.~\ref{eqn:less_equal_ours}.

Denote $\mathbf{R}_{\text{ck}}, \{\mathbf{D}^{m}_{\text{ck}}\}_{m = 1}^{M}, \{\mathbf{b}_{i, \text{ck}}^m\}_{i = 1, m = 1}^{N, M}$ as the optimal solution of CKM in Eqn.~\ref{eqn:ck_means}. A feasible solution of OCKM can be constructed by
\begin{align*}
\mathbf{R}_{\text{ock}} & = \mathbf{R}_{\text{ck}} \\
\mathbf{D}^{m, c}_{\text{ock}} &= 
\begin{cases}
\mathbf{D}_{\text{ck}}^{m} & c = 1\\
\mathbf{0} & c \ge 2
\end{cases}\\
\mathbf{b}_{i, \text{ock}}^{m, c} & = \mathbf{b}_{i, \text{ck}}^{m}~~  c \in \{1, \cdots, C\}.
\end{align*}
With the constructed parameters, the  objective function of OCKM remains the same with CKM, which proves the Eqn.~\ref{eqn:less_equal_ck_ock}.
\end{proof}

This theorem implies the proposed OCKM
can potentially achieve a lower distortion error
with the number of partitions $M$ and $K$ fixed.


\begin{theorem}
Under the optimal  solutions, we have,
\begin{align}
f^*_{\text{ock}, M', K, C} \le f_{\text{ck}, M, K}^{*}
\label{eqn:un_equal_subvector}
\end{align}
if $M' = M/C$ and $M$ is divisible by $C$.
\label{thm:code_length}
\end{theorem}

\begin{proof}
The basic idea is for the optimal solution of CKM, every consecutive $C$ sub codebooks and the binary representation are grouped to construct a feasible solution of OCKM with an equal objective function.

Specifically, the construction is
\begin{align*}
\mathbf{R}_{\text{ock}} &= \mathbf{R}_{\text{ck}} \\
\mathbf{D}_{\text{ock}}^{p, q} &= \begin{bmatrix}
\mathbf{0}_{(q - 1)S \times K} \\
\mathbf{D}_{\text{ck}}^{(p - 1)C + q}\\
\mathbf{0}_{(C - q)S \times K}
\end{bmatrix} \\
\mathbf{b}_{i, \text{ock}}^{p, q} & = \mathbf{b}_{i, \text{ck}}^{(p - 1)C + q}, 
\end{align*}
where $\mathbf{0}_{a \times b}$ is a matrix of size $a\times b$ with all entries being $0$, 
and $p\in \{1, \cdots, M'\}, q\in \{1, \cdots, C\}$.
\end{proof}

Take $C = 2$, $M = 2$ as an example.
The formulation of CKM is
\begin{align*}
\min ~~&f_{\text{ck}, 2, K} = \sum_{i = 1}^{N}\left\|\mathbf{x}_i - \mathbf{R}
\begin{bmatrix}
	\mathbf{D}^{1} 	& 	\mathbf{0}					\\
				\mathbf{0}	& \mathbf{D}^{2}
\end{bmatrix}
\begin{bmatrix}
\mathbf{b}_i^{1} \\
\mathbf{b}_i^{2}
\end{bmatrix}
\right\|_2^2 \\
\operatorname{s.t.}~~& \mathbf{R}^{T}\mathbf{R} = \mathbf{I} \\
& \mathbf{b}_i^{m} \in \{0, 1\}^{K}\\
& \|\mathbf{b}_i^{m}\|_1 = 1
\end{align*}
Let $\mathbf{R}_{\text{ck}}$,
$\{\mathbf{D}^{m}_{\text{ck}}\}_{m = 1}^{2}$,
$\{\mathbf{b}_{i, \text{ck}}^{m}\}_{i = 1, m = 1}^{N, 2}$,
be the optimal solutions of CKM.
Then
\begin{align*}
\mathbf{R}_{\text{ock}} &= \mathbf{\mathbf{R}}_{\text{ck}} \\
\mathbf{D}_{\text{ock}}^{1, 1} &=
\begin{bmatrix}
\mathbf{D}_{\text{ck}}^{1} \\
\mathbf{0}
\end{bmatrix}\\
\mathbf{D}_{\text{ock}}^{1, 2} & =
\begin{bmatrix}
\mathbf{0} \\
\mathbf{D}_{\text{ck}}^{2}
\end{bmatrix} \\
\mathbf{b}_{\text{ock}}^{1, c} &= \mathbf{b}_{\text{ck}}^{c}~~c \in \{1, 2\}
\end{align*}
will be feasible for the problem of OCKM, i.e.
\begin{align*}
\min ~~&f_{\text{ock}, 1, K, 2} = \sum_{i = 1}^{N}{\left\|\mathbf{x}_i - \mathbf{R}
\begin{bmatrix}
\mathbf{D}^{1, 1} & \mathbf{D}^{1, 2}
\end{bmatrix}
\begin{bmatrix}
\mathbf{b}_i^{1, 1}\\
\mathbf{b}_i^{1, 2}
\end{bmatrix}
\right\|_2^2} \\
\operatorname{s.t.} ~~&\mathbf{R}^{T}\mathbf{R} = \mathbf{I} \\
&\mathbf{b}_i^{1, c} \in \{0, 1\}^{K}\\
& \|\mathbf{b}_i^{1, c}\|_1 = 1
\end{align*}
and they have identical objective function values.

In Theorem \ref{thm:code_length}, the code length of both approaches is $M / C \times C \times \log_2(K) = M\log_2(K)$, which
ensures the distortion error of OCKM is not larger than that of CKM with the same code length.

Theorem~\ref{thm:same_number_subvector} and
Theorem~\ref{thm:code_length}
guarantee the advantages of our OCKM with
multiple sub codebooks over the approach
with single sub codebook.

\subsection{Inequality Constraints or Equality Constraints}
One may expect to replace the equality constraint $\|\mathbf{b}_i^{m, c}\|_1 = 1$
in Eqn.~\ref{eqn:ock_means} as the inequality, i.e.
\begin{align}
\|\mathbf{b}_i^{m, c}\|_1 \le 1.
\end{align}
This can potentially give a lower distortion under the same $M$ and $K$.
However, under the same code length, this inequality constraint cannot be better than the equality constraints.

For the inequality case,
there are $K + 1$ different values for $\mathbf{b}_{i, \text{inequality}}^{m}$,
i.e.
$\|\mathbf{b}_{i, \text{inequality}}^{m, c}\|_1 = 0$, or $1$.
The subscripts $\text{equality}$ and $\text{inequality}$ are used for the problem with the equality constraint and that with the inequality constraint, respectively.
Then, the code length is $MC\log_2(K + 1)$.

With the same code length, the equality case can consume $K + 1$ sub codewords on each subvector. The size of $\mathbf{D}_{\text{equality}}^{m, c}$ is $S\times (K + 1)$, and the size of $\mathbf{b}_{i, \text{equality}}^{m, c}$ is $(K + 1) \times 1$.

From any feasible solution of the inequality case, we can derive the feasible solution of the equality case with the same objective function value, i.e.
\begin{align*}
\mathbf{R}_{\text{equality}} & = \mathbf{R}_{\text{inequality}} \\
\mathbf{D}_{\text{equality}}^{m, c} & =
\begin{bmatrix}
\mathbf{D}_{\text{inequality}}^{m, c}, \mathbf{0}_{S\times 1}
\end{bmatrix} \\
\mathbf{b}_{\text{equality}}^{m, c} & = \begin{cases}
\begin{bmatrix}
\mathbf{b}_{\text{inequality}}^{m, c}\\
0
\end{bmatrix} &\text{if } \|\mathbf{b}_{\text{inequality}}^{m, c}\|_1 = 1 \\
\begin{bmatrix}
\mathbf{0}_{K\times 1} \\
1
\end{bmatrix} &\text{if }
\|\mathbf{b}_{\text{inequality}}^{m, c}\|_1 = 0.
\end{cases}
\end{align*}

In the equality case,
the last sub codeword
is enforced to
be $\mathbf{0}_{S\times 1}$,
and
the other sub codewords
are filled by the one in the inequality case.
If $\mathbf{b}^{m, c}_{\text{inequality}}$ is all $0$s,
the entry of $\mathbf{b}_{\text{equality}}^{m, c}$
corresponding to the last sub codeword is set as $1$,
or follows $\mathbf{b}^{m, c}_{\text{inequality}}$.
This can ensure the multiplication
$\mathbf{D}_{\text{equality}}^{m, c}\mathbf{b}_{\text{equality}}^{m, c}$ equals
$\mathbf{D}_{\text{inequality}}^{m, c}\mathbf{b}_{\text{inequality}}^{m, c}$.

The objective function value remains the same, while with the optimal solution the equality case may obtain a lower distortion.

\begin{algorithm}[t]
\caption{Optimization of OCKM}
\label{alg:opt}
\begin{algorithmic}[1]
    \REQUIRE
        $\{\mathbf{x}_i\}_{i = 1}^{N}$, $M$
    \ENSURE
        $\mathbf{R}$, $\{\mathbf{D}^{m, c}\}_{m = 1, c = 1}^{M, C}$, and $\{\mathbf{b}_i^{m, c}\}_{i = 1,m = 1,c = 1}^{N, M, C}$
    \STATE
    	$\mathbf{R} = \mathbf{I}$
    \STATE
        Randomly initialize $\{\mathbf{D}^{m, c}\}_{m = 1, c = 1}^{M, C}$ from the data set.
    \STATE
         Update $\{\mathbf{b}_i^{m, c}\}_{i = 1, m = 1, c = 1}^{N, M, C}$ by Alg.~\ref{alg:opt_ck_b}
    \WHILE{!converged}
    	\STATE
        Update $\mathbf{R}$
        \STATE
        Update $\{\mathbf{D}^{m, c}\}_{m = 1, c = 1}^{M, C}$
        \FOR{$i = 1 : N$}
        	\FOR{$m = 1 : M$}
	        	\STATE
	        		Get ${\mathbf{new\_\hat{b}}}_i^{m, c}$ from Alg.~\ref{alg:opt_ck_b}
	        	\IF{$g_{\text{ock}}({\mathbf{new\_\hat{b}}}_i^{m}) < g_{\text{ock}}(\hat{\mathbf{b}}_i^{m})$}
	        	\STATE
	        		$\hat{\mathbf{b}}_i^{m} = {\mathbf{new\_\hat{b}}}_i^{m}$
	        	\ENDIF
        	\ENDFOR
        \ENDFOR
    \ENDWHILE
\end{algorithmic}
\end{algorithm}

\subsection{Implementation}
In OCKM and ECKM, there are three kinds of optimizers: rotation matrix $\mathbf{R}$,
sub codebooks $\mathbf{D}^{m}$ or $\mathbf{D}^{m, c}$,
and $\mathbf{b}_i^{m}$ or $\mathbf{b}_i^{m, c}$.
In our implementation, $\mathbf{R}$ is initialized as the identity matrix $\mathbf{I}$.
The sub codebook $\mathbf{D}^{m}$ and $\mathbf{D}^{m, c}$ are initialized by randomly choosing the data on the corresponding subvector.

The solution of $\mathbf{R}$,  $\mathbf{D}^{m}$ and $\mathbf{D}^{m, c}$
are optimal in the iterative optimization process,
but the solution of $\mathbf{b}_i^m$ and $\mathbf{b}_i^{m, c}$
are sub optimal.
To guarantee that the objective function value is non-increasing in the iterative coordinate descent algorithm,
we update $\mathbf{b}_i^m$ or $\mathbf{b}_i^{m, c}$
only if the codes of Alg.~\ref{alg:ick_code} or Alg.~\ref{alg:opt_ck_b}
can provide a lower distortion error.
The whole algorithm of OCKM is shown in Alg.~\ref{alg:opt} and the one of ECKM can be similarly obtained.

The distortion errors of OCKM
with different numbers of iterations
are shown in
Fig.~\ref{fig:convergence} on SIFT1M (Sec.~\ref{sec:dataset} for the dataset description),
and we use $100$ iterations through all the experiments.
The optimization scheme is fast and
for instance on the training set of SIFT1M, the time cost of each iteration is about $4.2$ seconds in our implementations. (All the experiments are conducted on a server with an Intel Xeon 2.9GHz CPU.)

\subsection{Distance Approximation for ANN search}
\label{sec:ann_search}
In this subsection, we discuss the methods of the Euclidean ANN search by OCKM, and analyze the query time.
Since ECKM is a special case of OCKM, we only discuss OCKM.

Let $\mathbf{q}\in \mathbb{R}^{D}$ be the query point. The approximate distance to $\mathbf{x}_i$ encoded as ${\hat{\mathbf{b}}_i}^T \triangleq \begin{bmatrix}
{{}\hat{\mathbf{b}}_i^{1}}^{T} & \cdots & {{}\hat{\mathbf{b}}_i^{M}}^{T}
\end{bmatrix}$ is
\begin{align}
 & \text{distAD}(\mathbf{q}, \hat{\mathbf{b}}_i) \label{eqn:asd} \\
= & \|\mathbf{q} -
\mathbf{R}
\hat{\mathbf{D}}
\hat{\mathbf{b}}_i\|_2^2 \notag \\
= & \|\mathbf{q}\|_2^2 - 2 \sum_{m = 1}^{M} \sum_{c = 1}^{C}{{\mathbf{z}^m}^{T}(\mathbf{D}^{m, c}\mathbf{b}_i^{m, c}) + \|\hat{\mathbf{D}}\hat{\mathbf{b}}_i\|_2^2} \notag\\
\propto & \frac{1}{2}\|\mathbf{q}\|_2^2 - \sum_{m = 1}^{M}\sum_{c = 1}^{C}{{\mathbf{z}^m}^{T}(\mathbf{D}^{m, c}\mathbf{b}_i^{m, c}) + \frac{1}{2}\|\hat{\mathbf{D}}\hat{\mathbf{b}}_i\|_2^2},
\label{eqn:app_dist}
\end{align}
where $\mathbf{z}^{m}$ is the $m$-th subvector of $\mathbf{R}^{T} \mathbf{q}$.

\begin{figure}[t]
\centering
\includegraphics[width = 0.75\linewidth]{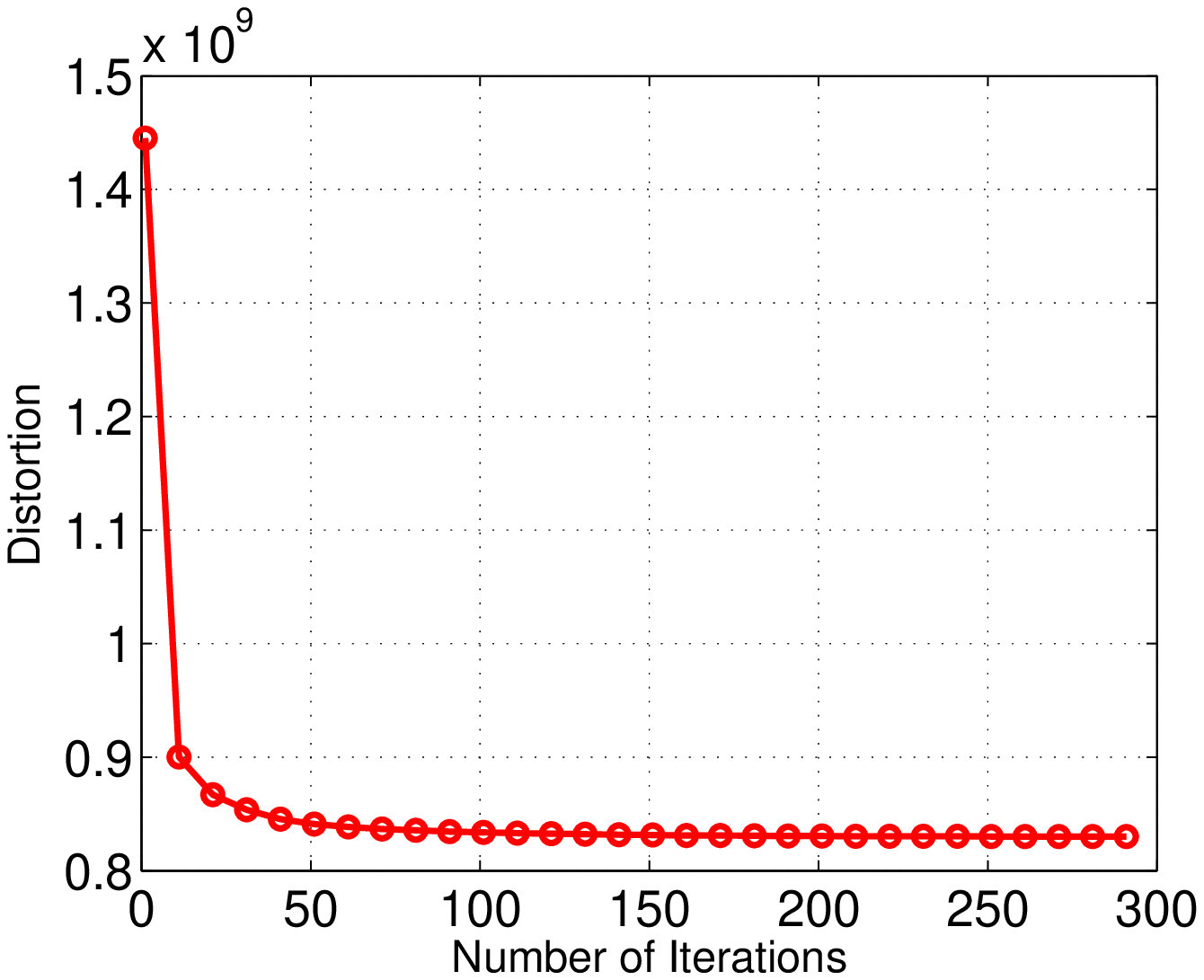}
\caption{Distortion vs the number of iterations on the training set of SIFT1M with $M = 8$, $K = 256$ and $C = 2$.}
\label{fig:convergence}
\end{figure}

The first item $\|\mathbf{q}\|_2^2 / 2$ is constant with all the database points and can be ignored in comparison.
The third item $\|\hat{\mathbf{D}}\hat{\mathbf{b}}_i\|_2^2 / 2$ is
independent of the query point.
Thus,
it is precomputed once as the lookup table
for all the quires.
This precomputation cost is not low compared with the linear scan cost for a single query,
but is negligible for a large amount of queries
which is the case in real applications.
Moreover,
this term is computed only using the binary code $\hat{\mathbf{b}}_i$
and
no access to the original $\mathbf{x}_i$ is required.
For the second item, we can pre-compute $\{-{\mathbf{z}^m}^{T}\mathbf{d}_{k}^{m, c}\}_{k = 1, m = 1, c = 1}^{K, M, C}$ and store it as the lookup tables.
Then there are  $MC + 1$ table lookups
and  $M C + 1$ addition operations
to calculate the distance.
The $1$ corresponds to the third item of Eqn.~\ref{eqn:app_dist}.

If the query point is also represented by the binary codes, denoted as $\hat{\mathbf{b}}_q$, we can recover $\mathbf{q}$
as
$\mathbf{q}'
\triangleq \mathbf{R}\hat{\mathbf{D}}\hat{\mathbf{b}}_q$.
Then the approximate distance to any database point will be identical with Eqn.~\ref{eqn:asd}, i.e.
\begin{align}
\text{distSD}(\hat{\mathbf{b}}_q, \hat{\mathbf{b}}_i) = \text{distAD}(\mathbf{q}', \hat{\mathbf{b}}_i).
\label{eqn:sd}
\end{align}

Eqn.~\ref{eqn:asd} is
usually
referred as the asymmetric distance
while Eqn.~\ref{eqn:sd} as the symmetric distance.
Since the
symmetric distance
encodes both the query
and the database points,
the accuracy is
generally lower than the
asymmetric distance,
which only encodes the
database points.

\noindent\textbf{Analysis of query time.}
We adopt an exhaustive search in which each database point is compared against the query point and the points with smallest approximate distances are returned.
The exhaustive search scheme is  fast in practice because each comparison only requires a few table lookups and additional operations.

Table~\ref{tbl:query_comparison} lists the code length and the comparison among PQ, CKM and our OCKM for exhaustive search.
Under the same code length, OCKM consumes only one more table lookup and one more addition than the others.
Considering the other computations in the querying, the differences of time cost are minor in practice.

Take $M_{\text{ck}} = 8$, $K = 256$, $C = 2$, $M_{\text{ock}} = 4$ as an example. The code length of OCKM and CKM are both $64$. The number of table lookups are $9$ for OCKM and $8$ for CKM.
With these configurations on SIFT1M data set, the exhaustive querying over 1 million database points costs about $24.3$ms
for OCKM
and $23.5$ms for CKM in our implementations.
Thus, the on-line query time is comparable with the state-of-the-art approaches, but
the proposed approach can potentially provide a better accuracy.

\begin{table}
\centering
\caption{Comparison in terms of the code length, the number of table lookups and the number of addition operations for exhaustive search.}
\label{tbl:query_comparison}
\begin{tabular}{c@{~}c@{~}c@{~}c}
\toprule
& OCKM & CKM~\cite{NorouziF13} & PQ~\cite{JegouDS11} \\
\midrule
Code Length & $M C\log_2(K)$ & $M\log_2(K)$ & $M\log_2(K)$ \\
\#(Table Lookups) & $M  C + 1$ & $M$ & $M$ \\
\#(Additions) & $M C + 1$ & $M$ & $M$ \\ \bottomrule
\end{tabular}
\end{table}

\section{Experiments}
\label{sec:exp}
\subsection{Settings}
\subsubsection{Datasets}
\label{sec:dataset}
Experiments are conducted on three
widely-used high-dimensional datasets:
SIFT1M~\cite{JegouDS11},
GIST1M~\cite{JegouDS11},
and SIFT1B~\cite{JegouDS11}.
Each dataset comprises of one training set (from which the parameters are learned),
one query set,
and one database (on which the search is performed).
SIFT1M provides $10^5$ training points,
$10^4$ query pints and $10^6$ database points with each point being a $128$-dimensional SIFT descriptor of local image structures around the feature points.
GIST1M provides $5\times 10^5$ training points, $10^3$ query points and $10^6$ database points with each point being a $960$-dimensional GIST feature.
SIFT1B is composed of $10^8$ training points, $10^4$ query
points and as large as $10^9$ database points. Following~\cite{NorouziF13}, we use the first $10^6$ training points on
the SIFT1B datasets. The whole training set is used on SIFT1M and GIST1M.

\subsubsection{Criteria}
ANN search is conducted to evaluate our proposed approaches, and three indicators are reported.

\begin{itemize}
\item{Distortion: }
distortion is referred here as the sum of the squared loss after representing each point as the binary codes or the indices of the sub codewords.
Generally speaking, the accuracy is better with a lower distortion.

\item{Recall: } recall is the proportion over all the queries where the true nearest neighbor falls within the top ranked vectors by the approximate distance.

\item{Mean overall ratio: }
mean overall ratio ~\cite{TaoYSK10} reflects the general quality of all top ranked neighbors.
Let $\mathbf{r}_i$ be the $i$-th nearest vector of a query $\mathbf{q}$ with the exact Euclidean distance,
and $\mathbf{r}_i^*$ be the $i$-th point of the ranking list by the approximate distance.
The rank-$i$ ratio, denoted by $R_i(\mathbf{q})$, is
\begin{align*}
R_i(\mathbf{q}) = \frac{\|\mathbf{q} - \mathbf{r}_i^*\|_2}{\|\mathbf{q} - \mathbf{r}_i\|_2}.
\end{align*}
The overall ratio is the mean of all $R_i(\mathbf{q})$, i.e.
\begin{align*}
\frac{1}{k}\sum_{i = 1}^{k}{R_i(\mathbf{q})}.
\end{align*}
The mean overall ratio is the mean of the overall ratios of all the queries.
When the approximate results are the same as exact search results,
the overall ratio will be  $1$.
The performance is better with a lower mean overall ratio.

\end{itemize}

\begin{figure}[ttt!]
\centering
\begin{tabular}{c@{}c}
\includegraphics[width = 0.48\linewidth]{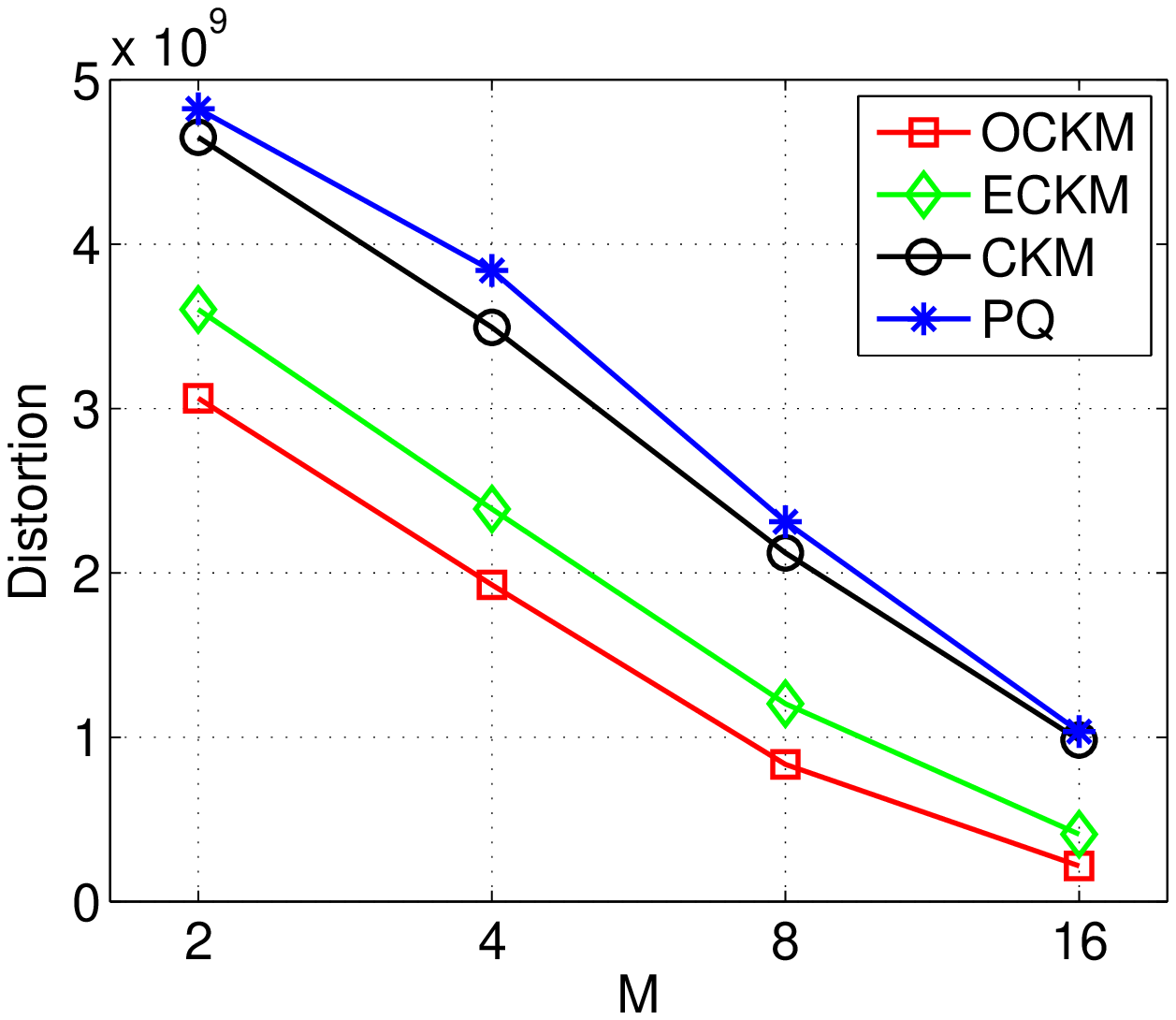}
&
\includegraphics[width = 0.48\linewidth]{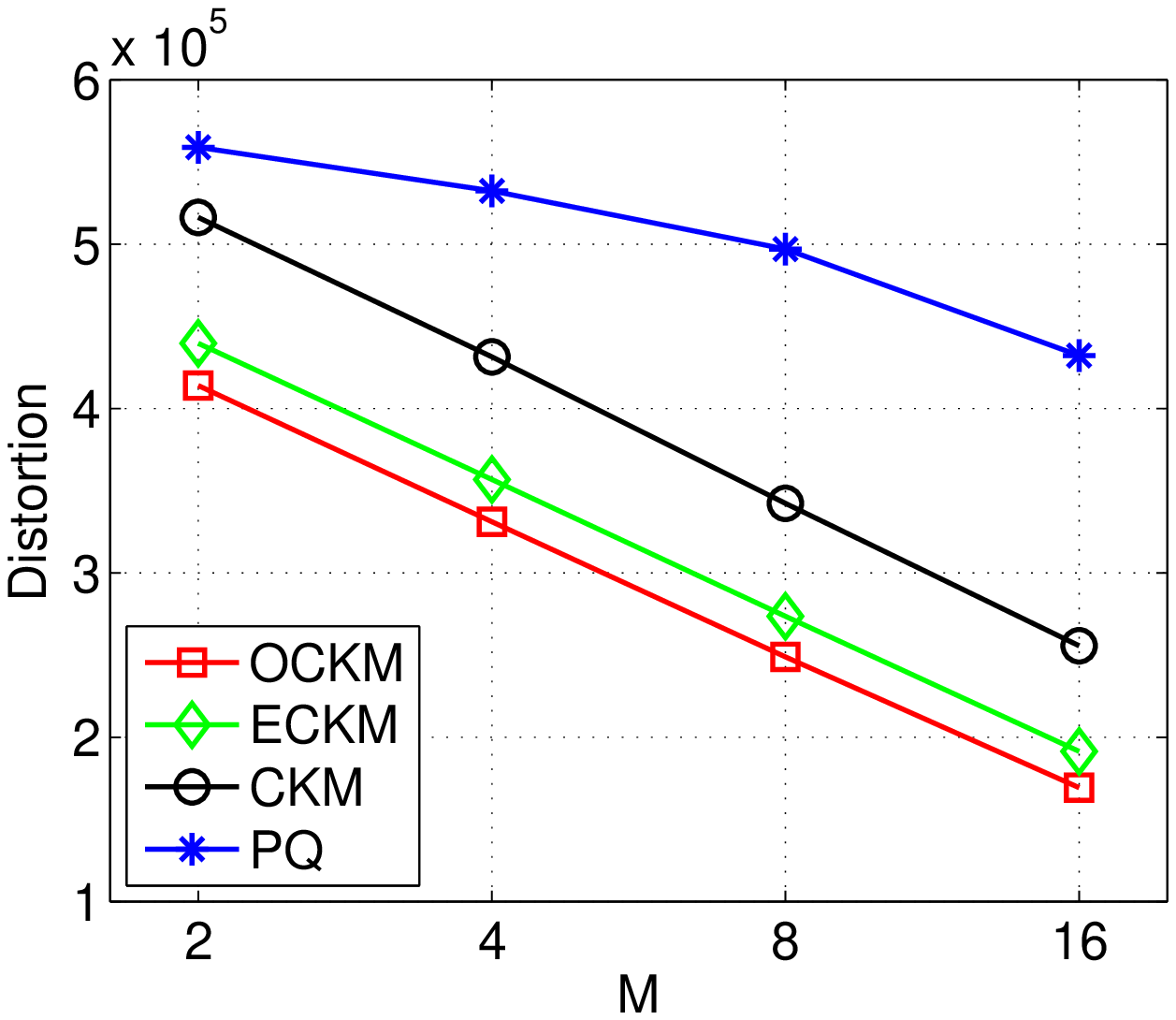} \\
(a) SIFT1M & (b) GIST1M
\end{tabular}
\caption{Distortion on the training set.}
\label{fig:distortion_train}
\end{figure}
\begin{figure}[t]
\begin{tabular}{c@{}c}
\includegraphics[width = 0.48\linewidth]{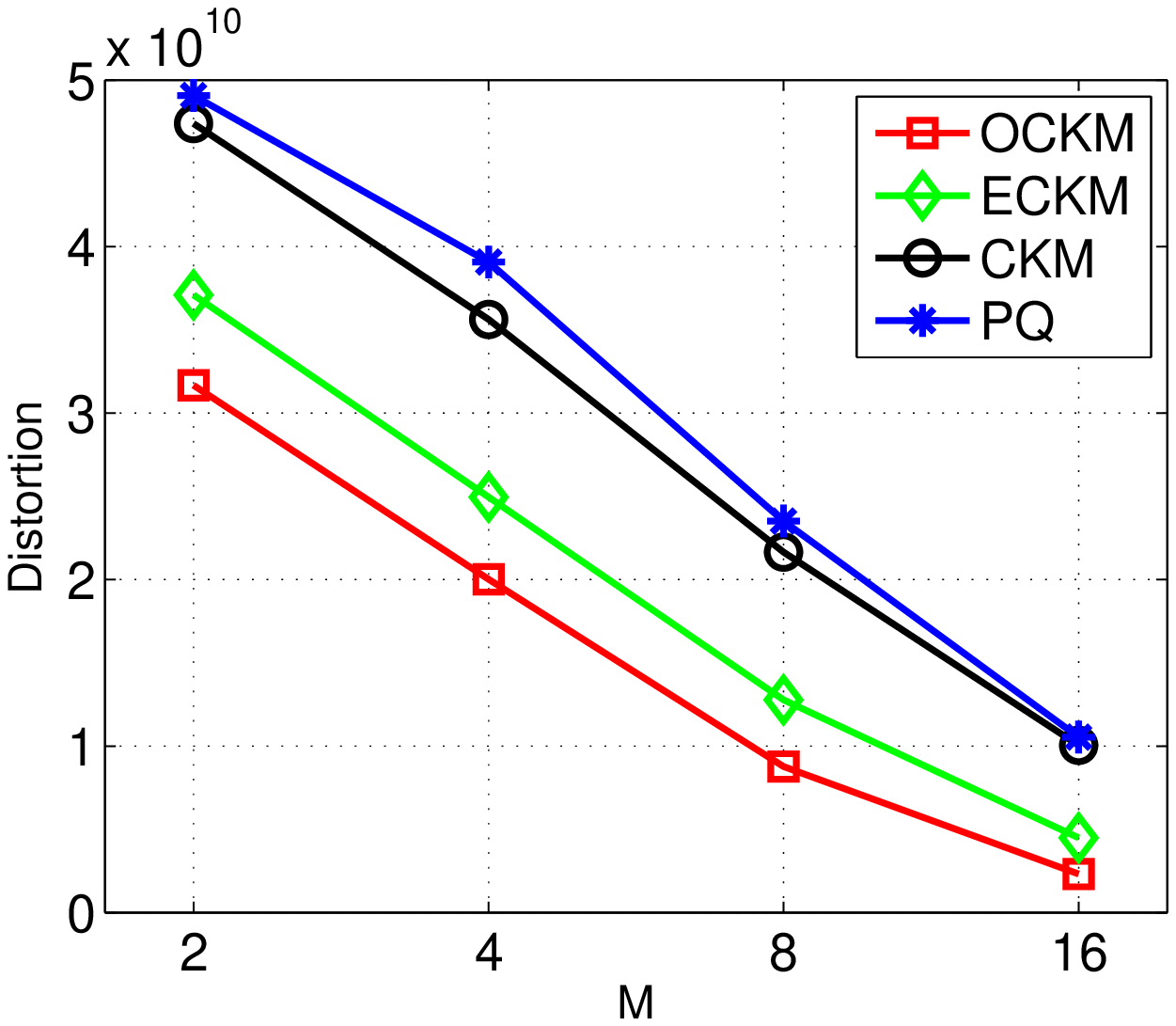}
&
\includegraphics[width = 0.48\linewidth]{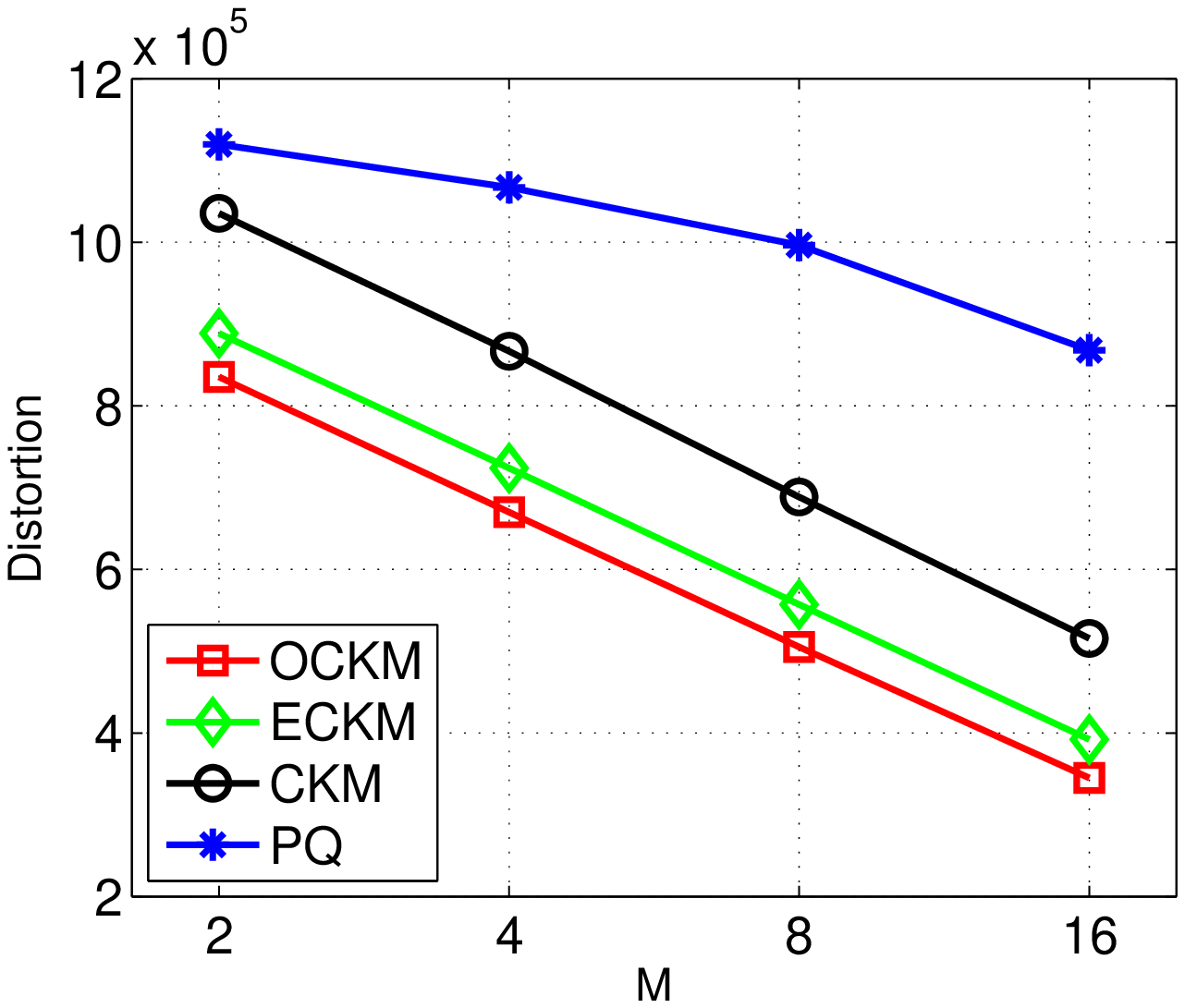} \\
(a) SIFT1M & (b) GIST1M
\end{tabular}
\caption{Distortion on the database set.}
\label{fig:distortion_base}
\end{figure}

\subsubsection{Approaches}
We compare our Optimized Cartesian $K$-Means (OCKM) with Product Quantization (PQ)~\cite{JegouDS11} and
Cartesian $K$-Means (CKM)~\cite{NorouziF13}.
Besides, the results of our extended Cartesian $K$-Means (ECKM) are also reported.
Following~\cite{NorouziF13},
we set $K = 256$ to make the lookup tables small and fit the sub index into one byte.

A suffix `-A' or `-S' is appended to the name of approaches to distinguish the asymmetric distance or the symmetric distance in ANN search.
For example, OCKM-A represents the database points are  encoded by OCKM, and the asymmetric distance is used to rank
all the database points.

We do not compare with other state-of-the-art hashing algorithms, such as spectral hashing (SH)~\cite{WeissTF08} and iterative quantization (ITQ) hashing~\cite{GongL11}, because
it is demonstrated PQ is superior over SH~\cite{JegouDS11} and  CKM is better than ITQ~\cite{NorouziF13}.

%
%

\subsection{Results}

\subsubsection{Comparison with the number of subvectors fixed}

The distortion errors on the training set and database set are illustrated in Fig.~\ref{fig:distortion_train} and Fig.~\ref{fig:distortion_base}, respectively.
From the two figures, our OCKM achieves the lowest distortion, followed by ECKM. This is because under the same $M$, both CKM and ECKM are the special case of OCKM, as discussed in 
Theorem~\ref{thm:same_number_subvector}.

\begin{figure}[t]
\centering
\begin{tabular}{c@{}c}
\includegraphics[width = \figSingleTwoWidth\linewidth]{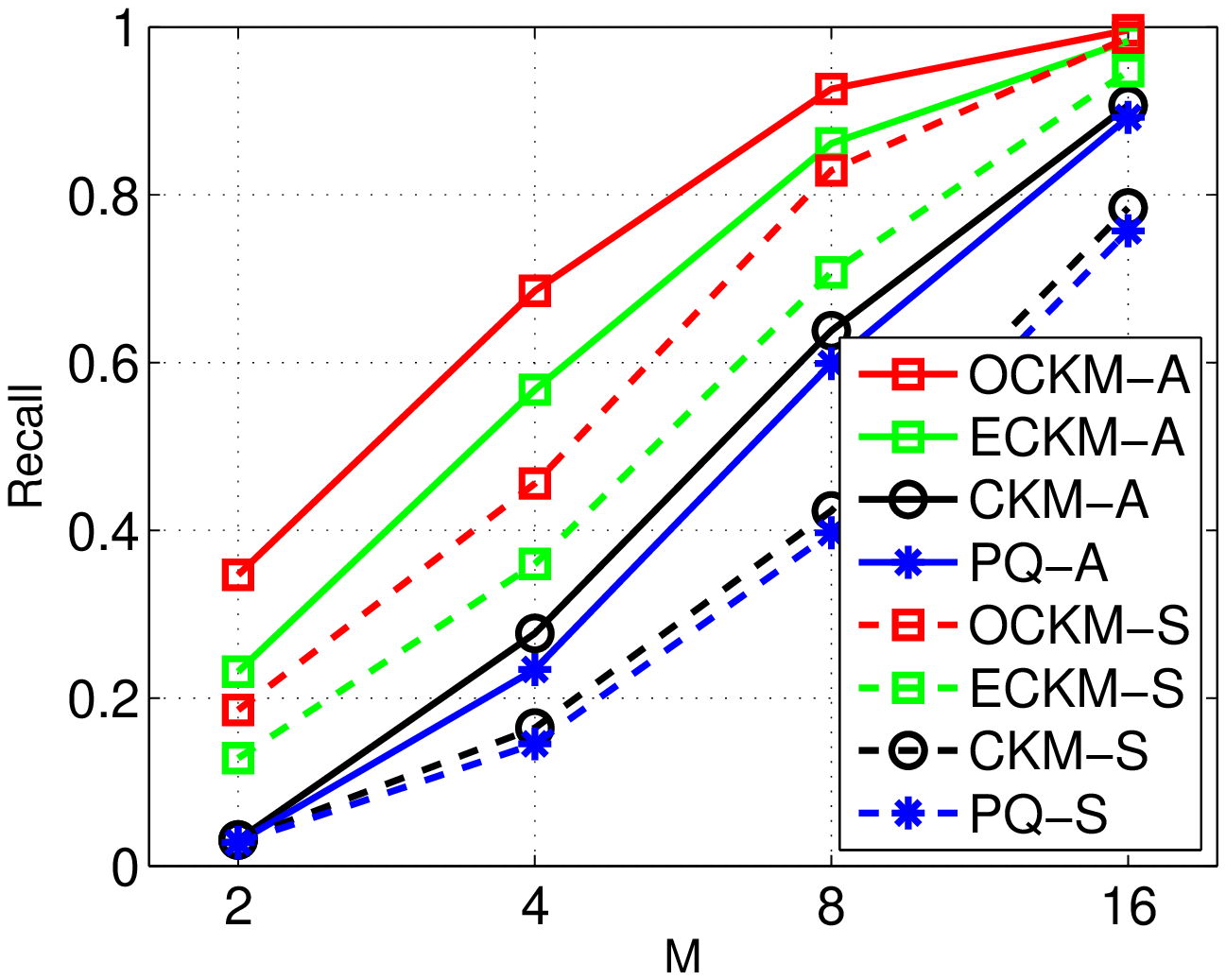}
&
\includegraphics[width = \figSingleTwoWidth\linewidth]{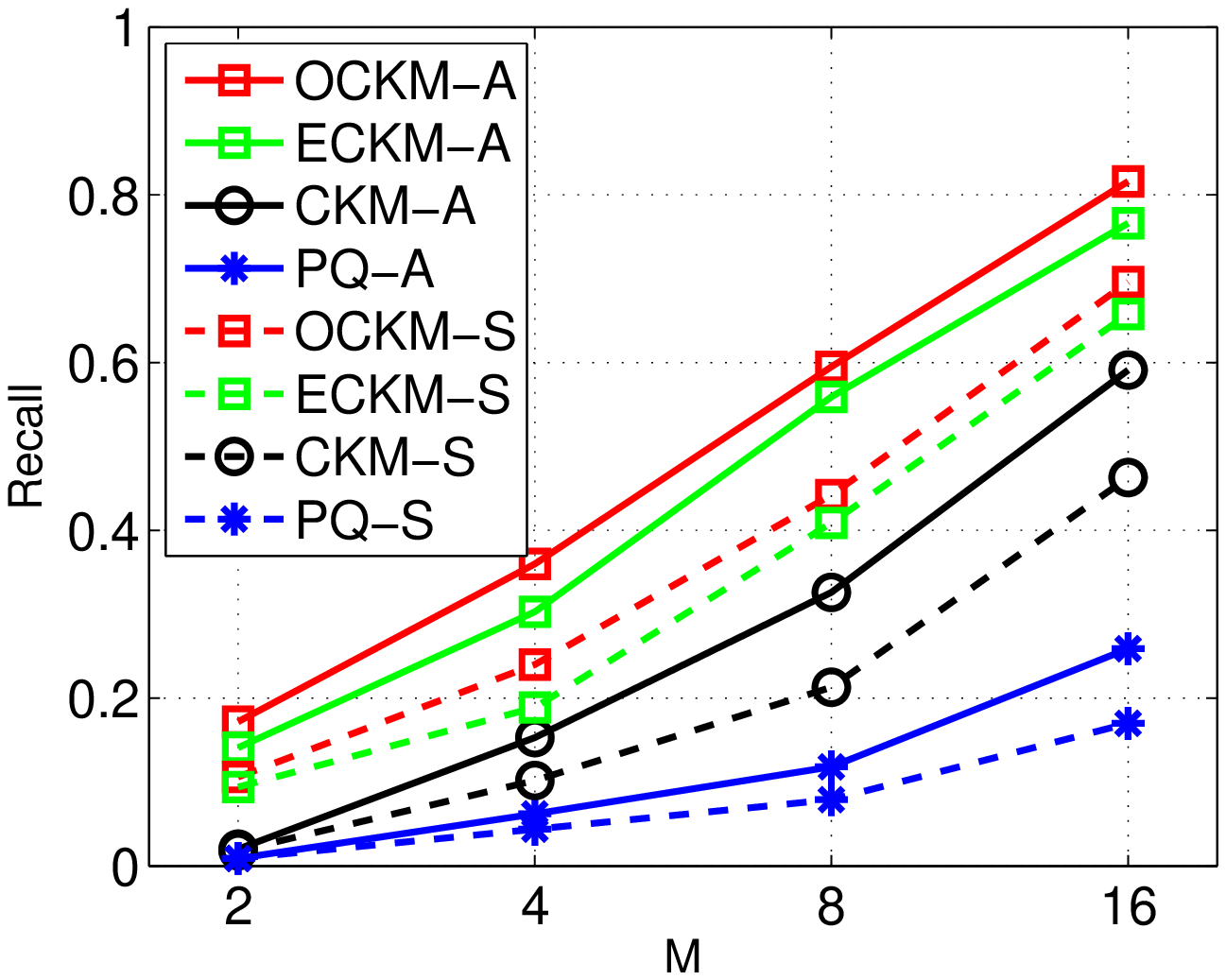} \\
(a) SIFT1M & (b) GIST1M
\end{tabular}
\caption{Recall for ANN search
at the $10$-th top ranked point.}
\label{fig:rec10}
\end{figure}

\begin{figure}
\centering
\begin{tabular}{c@{}c}
\includegraphics[width = \figSingleTwoWidth\linewidth]{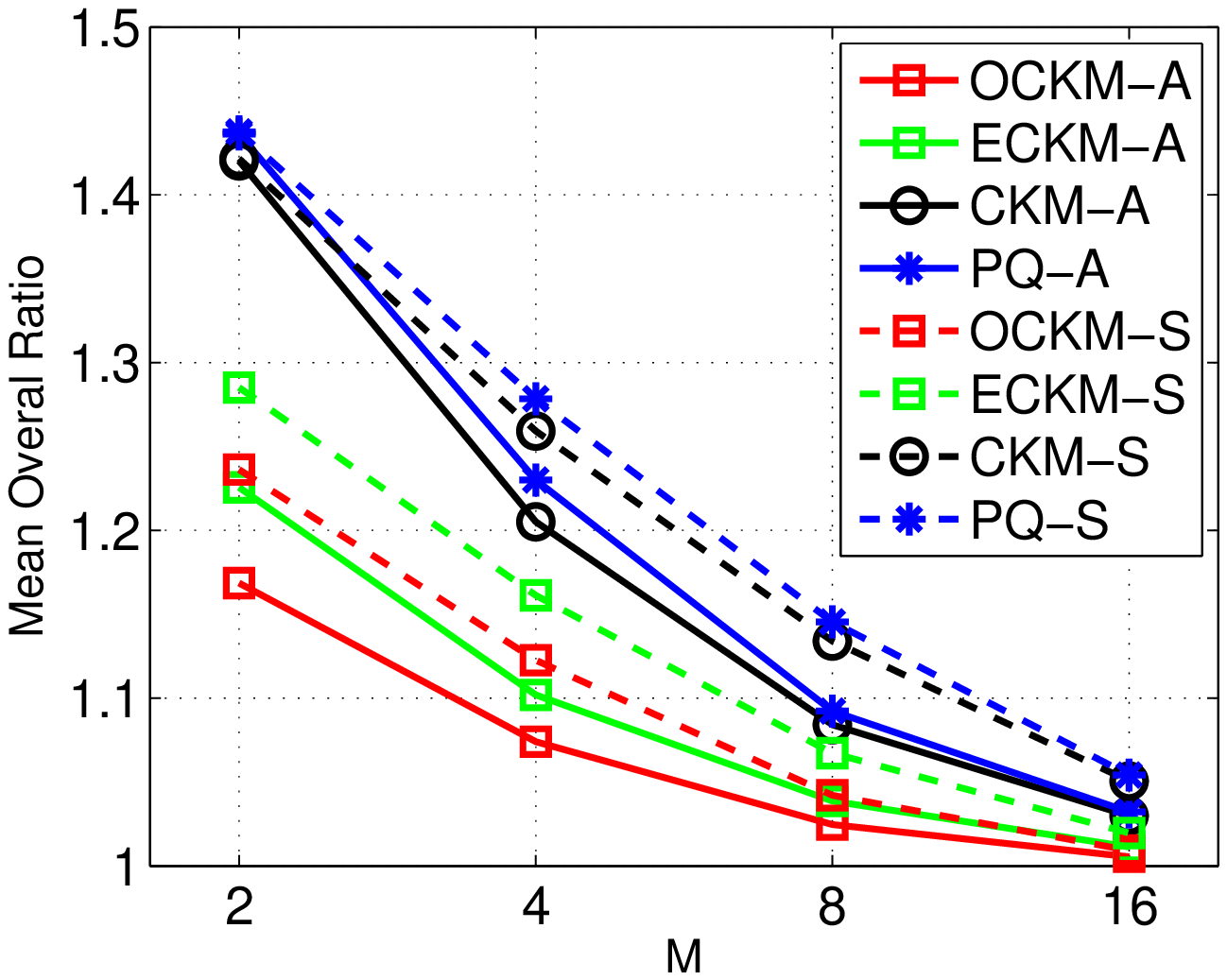} &
\includegraphics[width = \figSingleTwoWidth\linewidth]{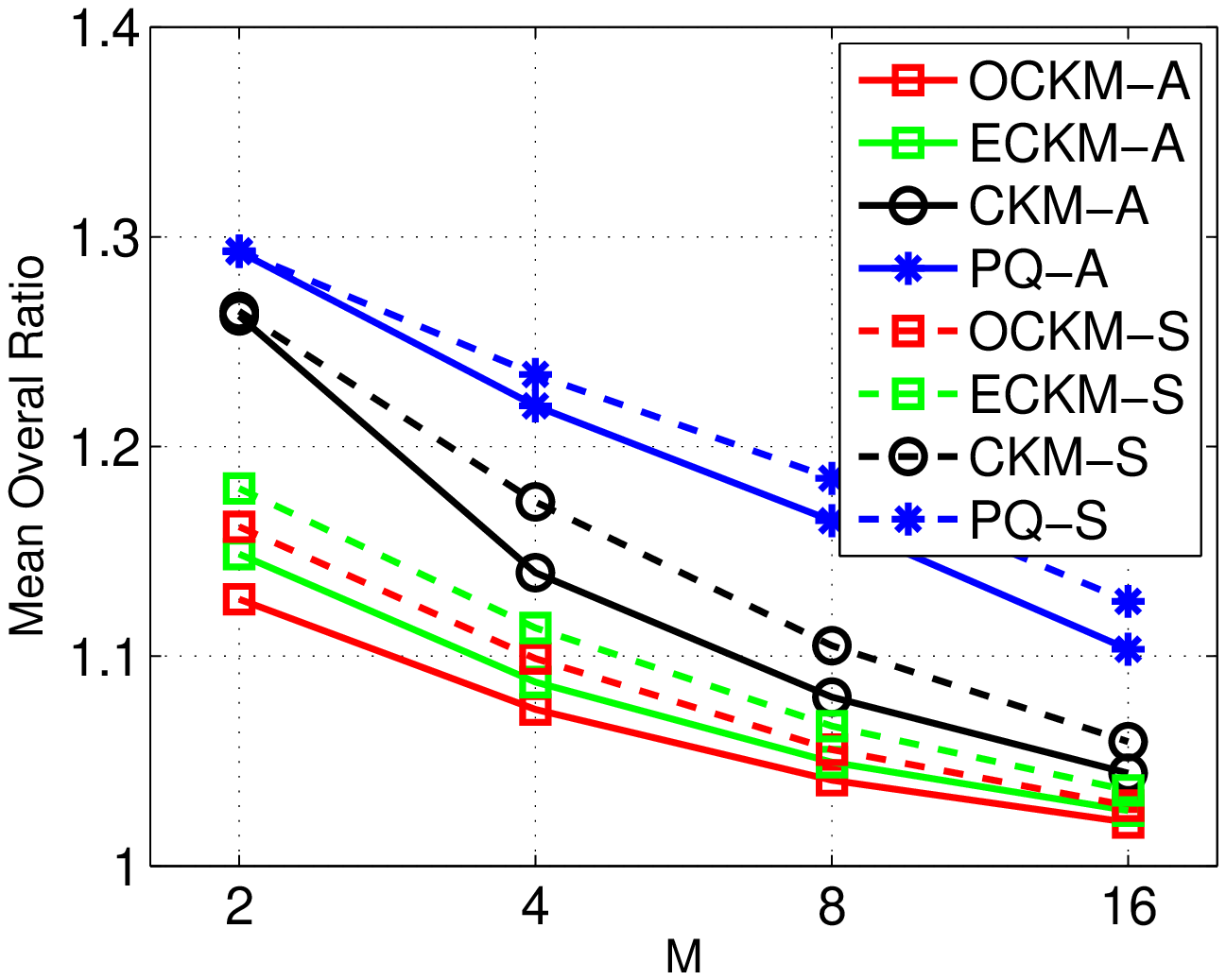} \\
(a) SIFT1M & (b) GIST1M
\end{tabular}
\caption{Mean overall ratio for ANN search
at the $10$-th top ranked point. }
\label{fig:mor10}
\end{figure}

Fig.~\ref{fig:rec10} and Fig.~\ref{fig:mor10}
show the recall and
the mean overall ratio
for ANN search
at
the $10$-th top ranked point, respectively.
With the same type of the approximate distance,
our approach OCKM
achieves the best performance: the highest recall and the lowest mean overall ratio.
With the lowest distortion errors demonstrated in
Fig.~\ref{fig:rec10} and Fig.~\ref{fig:mor10},
the OCKM is more accurate for encoding the data points.

\begin{figure*}[t]
\centering
\begin{tabular}{c@{}c@{}c}
32	& 	64	& 128 \\
\includegraphics[width = 0.32\linewidth]{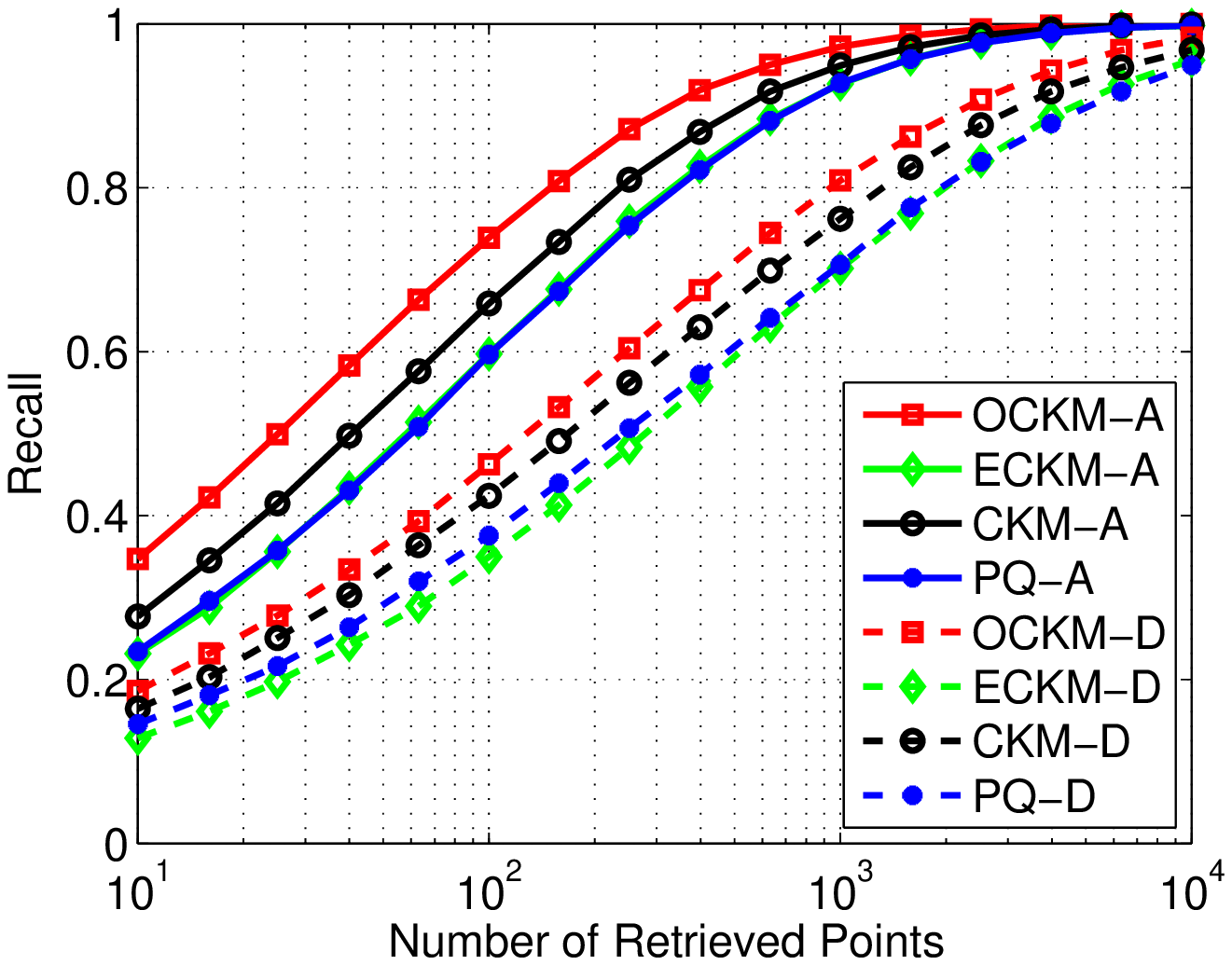} &
\includegraphics[width = 0.32\linewidth]{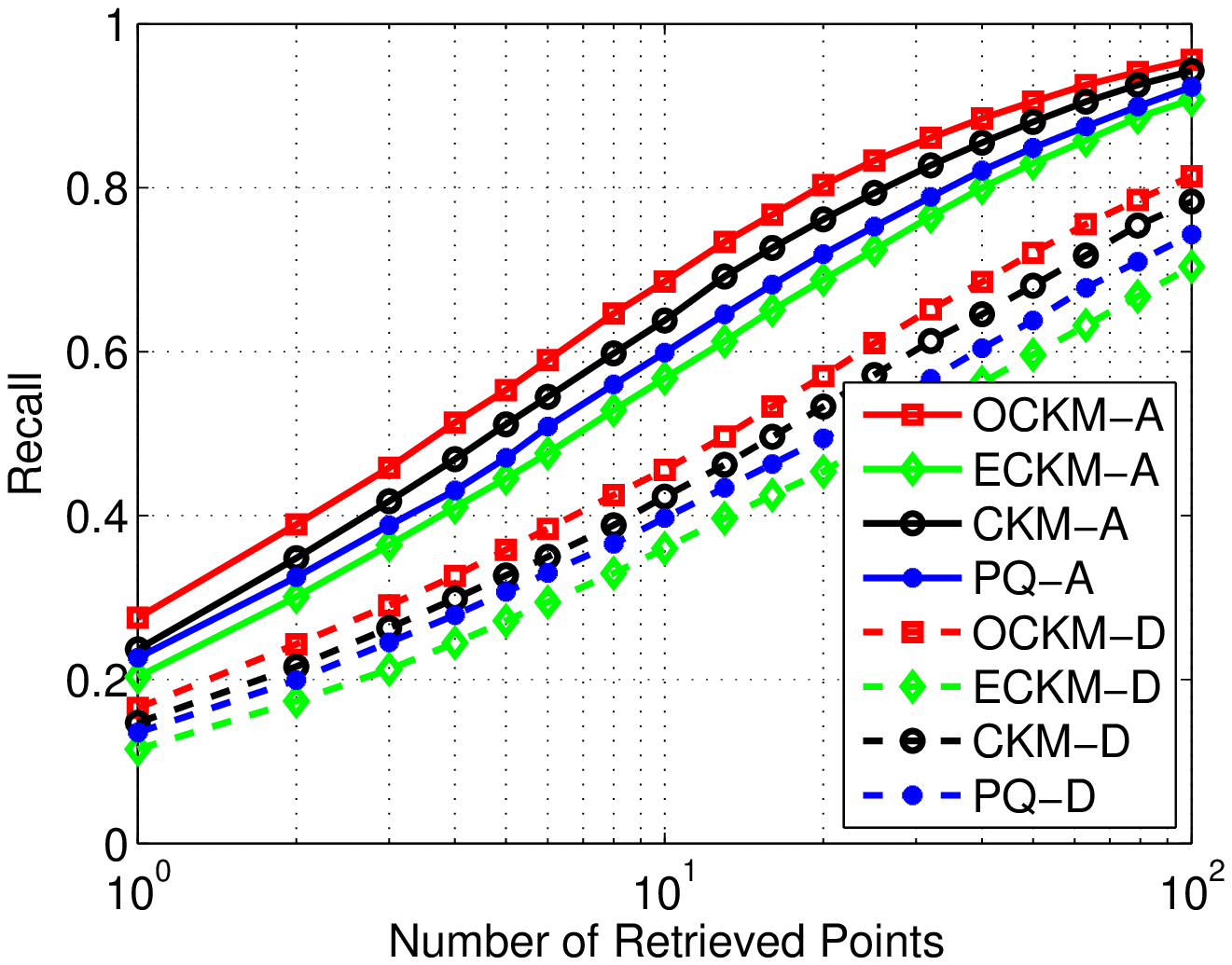} &
\includegraphics[width = 0.32\linewidth]{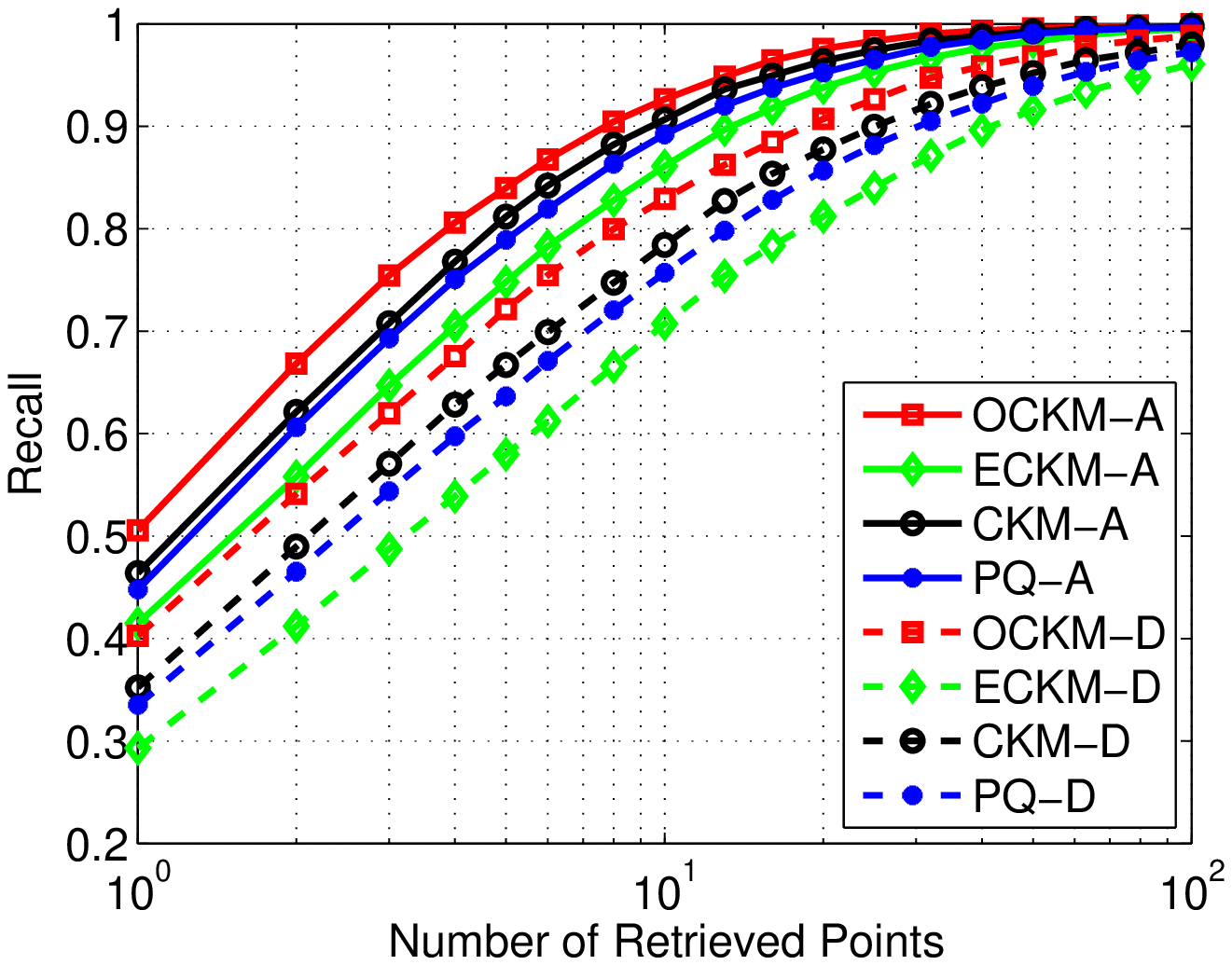} \\
(a)  & (b) & (c) \\
\includegraphics[width = 0.32\linewidth]{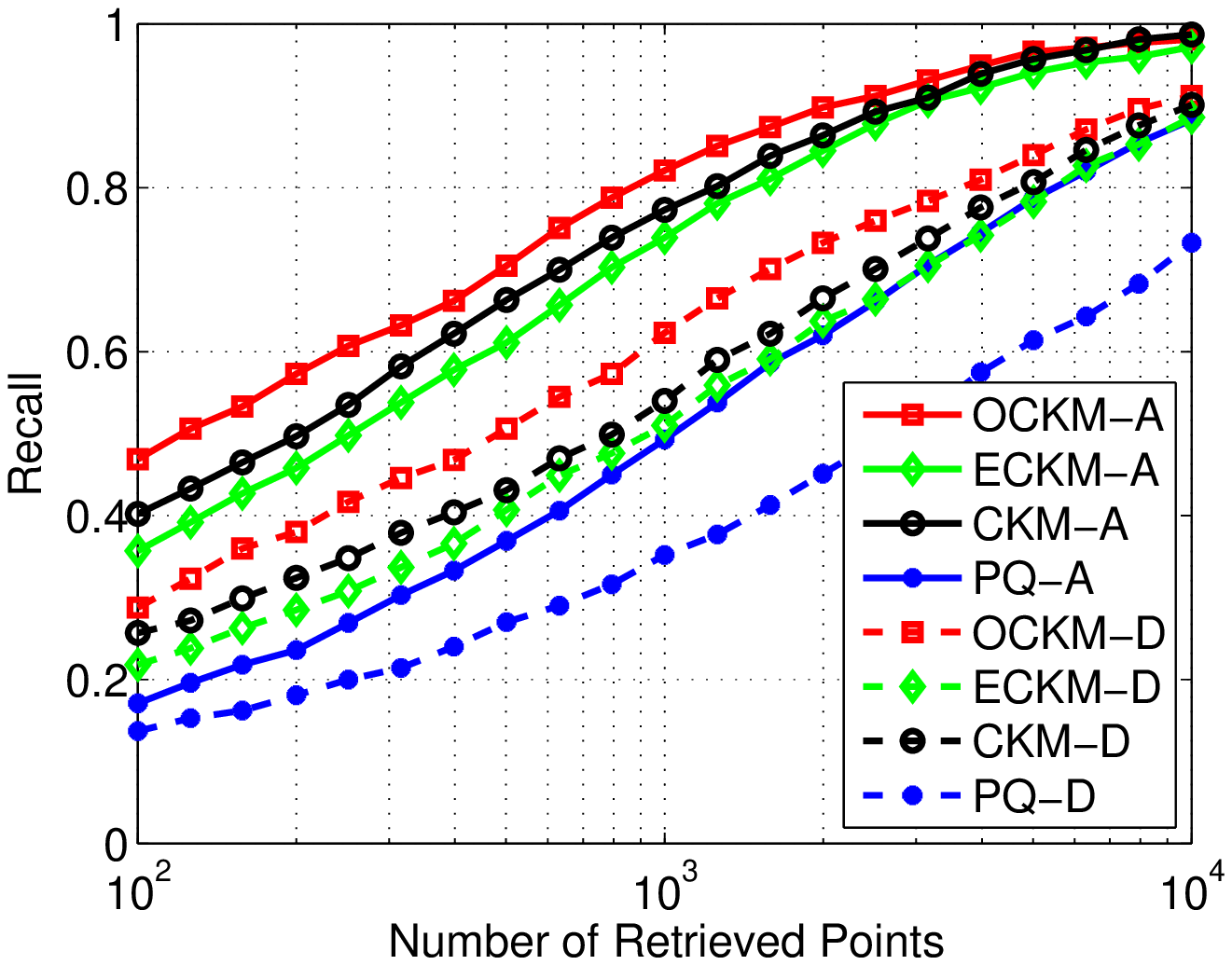} &
\includegraphics[width = 0.32\linewidth]{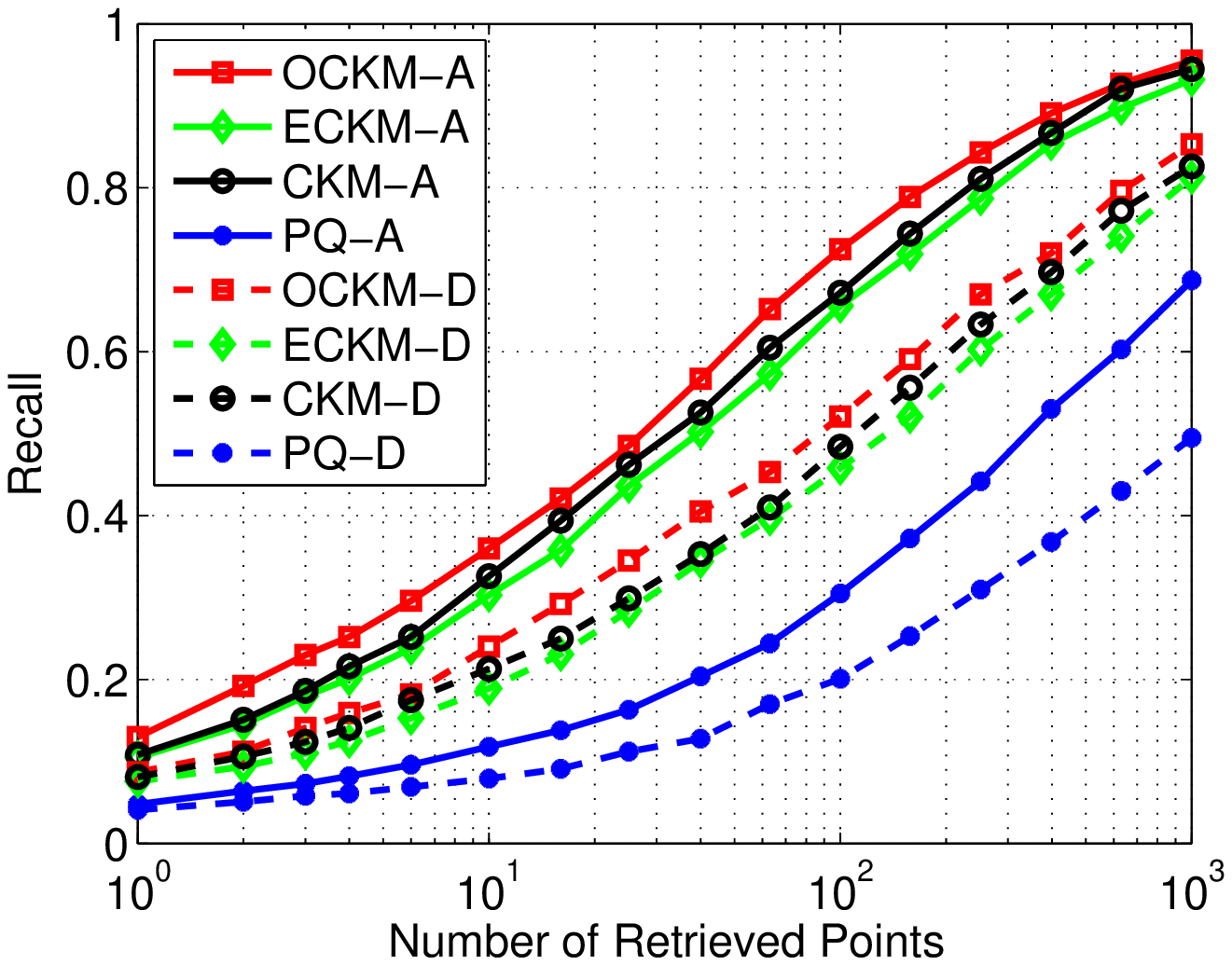} &
\includegraphics[width = 0.32\linewidth]{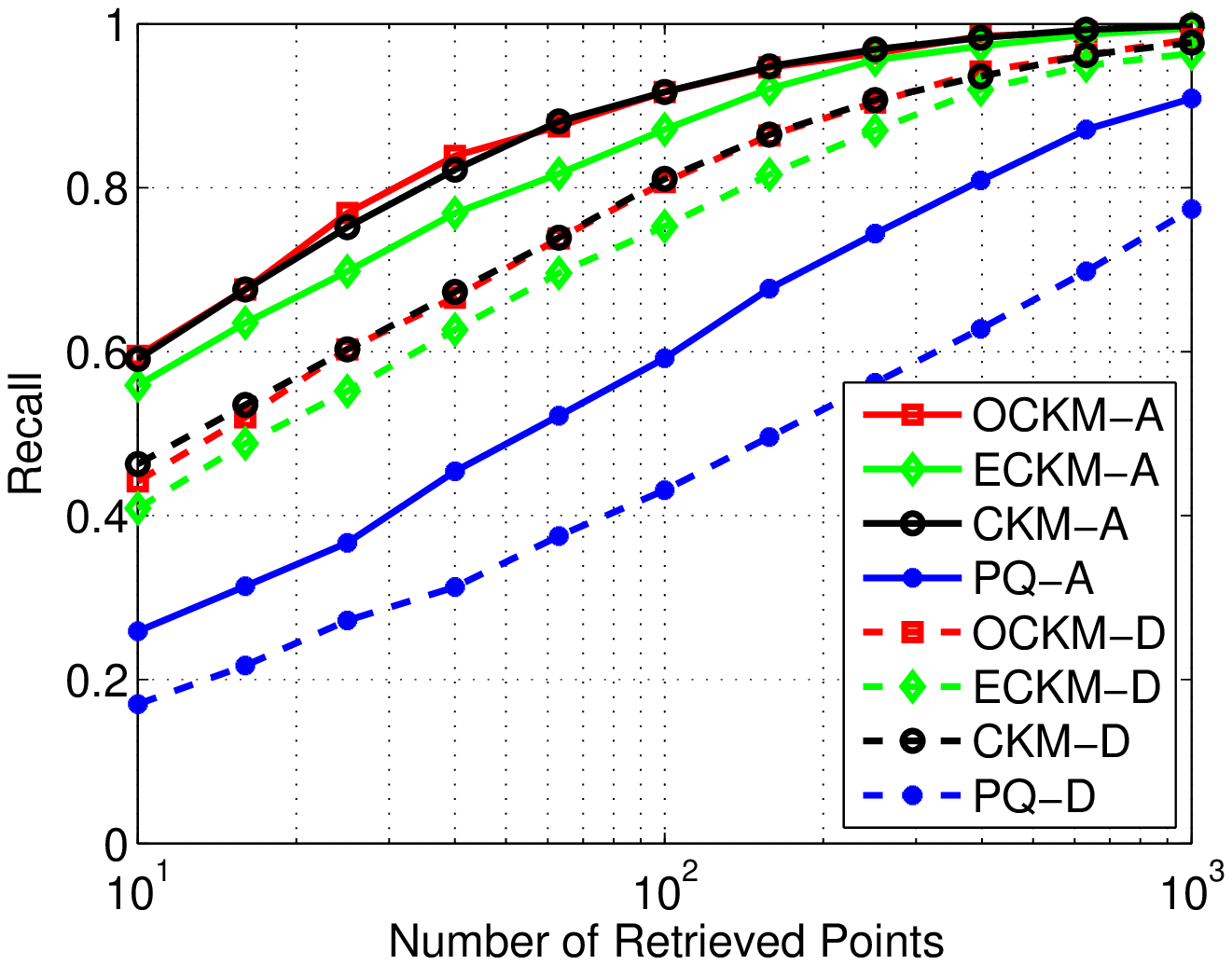} \\
(d) & (e) & (f) \\
\includegraphics[width = 0.32\linewidth]{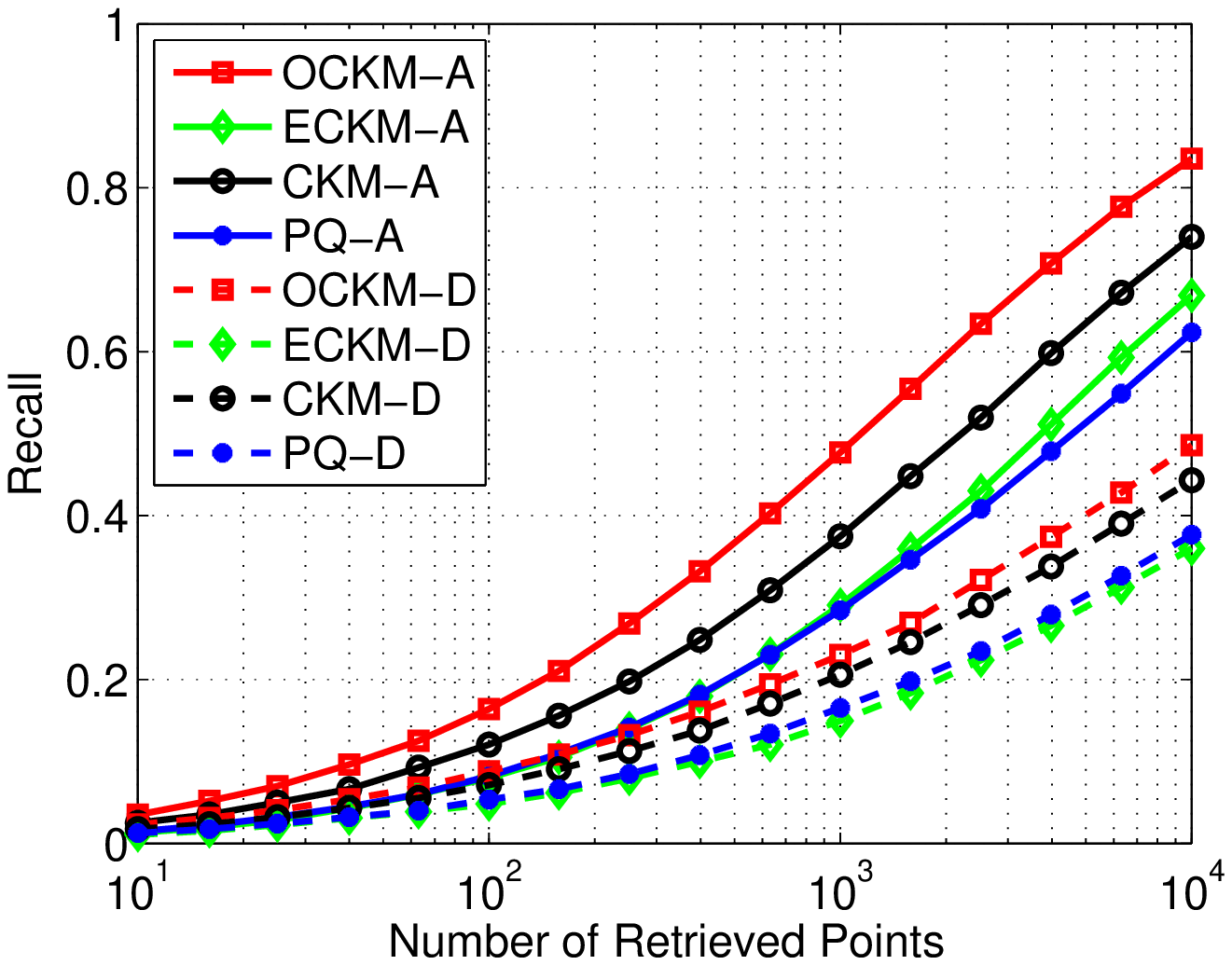} &
\includegraphics[width = 0.32\linewidth]{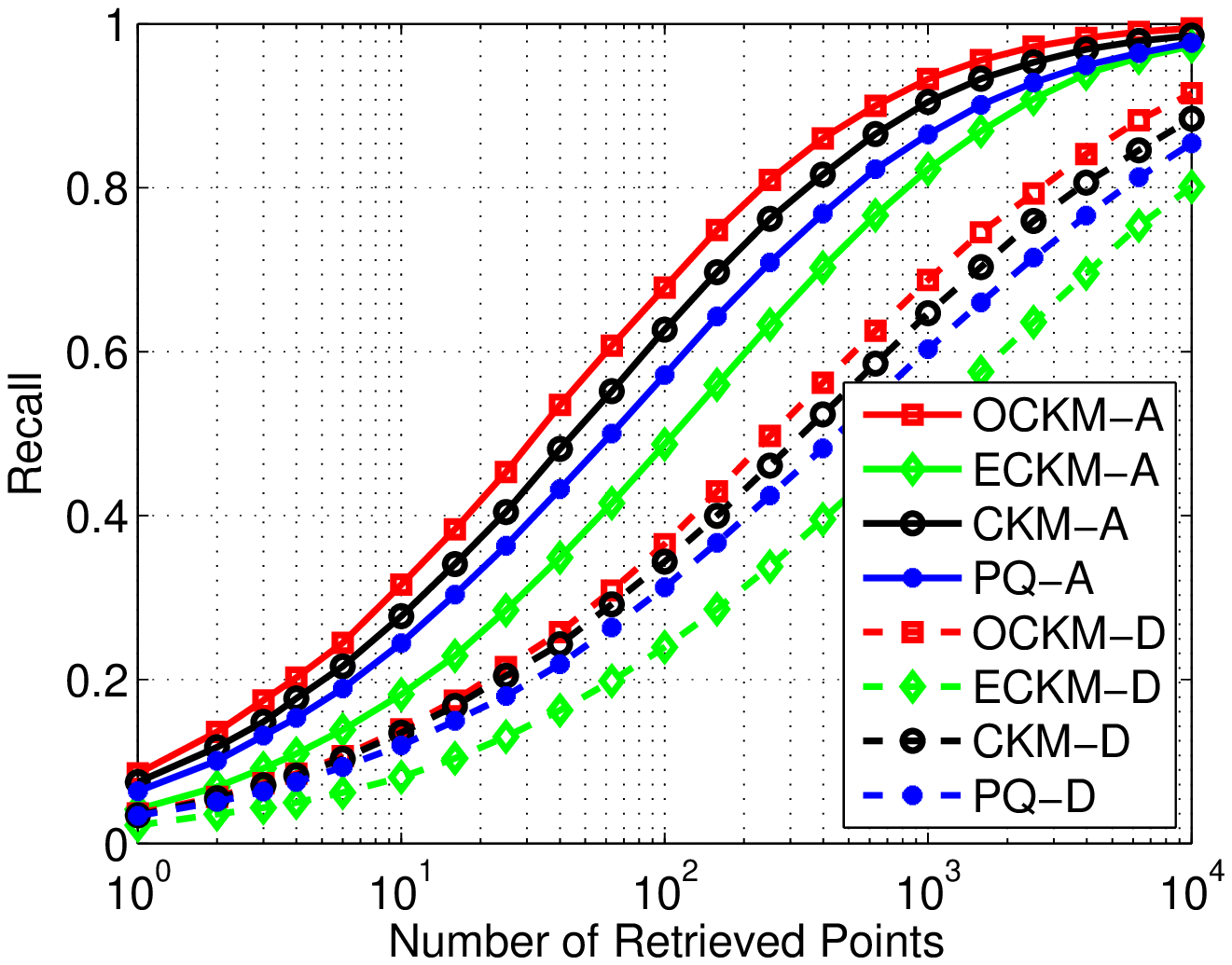} &
\includegraphics[width = 0.32\linewidth]{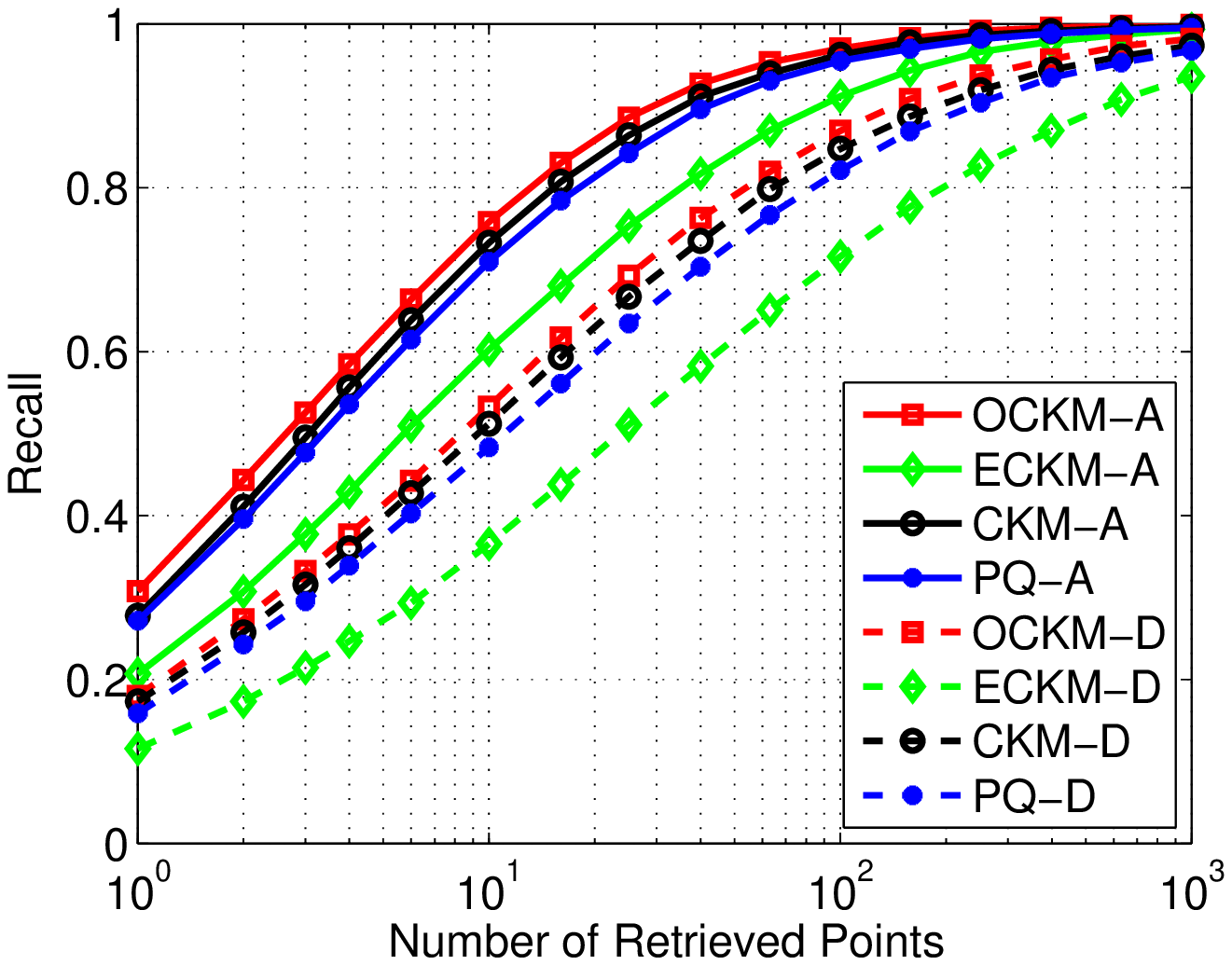} \\
(g) & (h) & (i)
\end{tabular}
\caption{Recall for ANN search. The first row corresponds to SIFT1M; the second to GIST1M; and the third to SIFT1B.
The code lengths are $32$, $64$ and $128$ from
the left-most column to the right-most.}
\label{fig:both_rec_32_64_128}
\end{figure*}

\subsubsection{Comparison with the code length fixed}
We use $M_{\text{ock}}$, $M_{\text{eck}}$, $M_{\text{ck}}$, $M_{\text{pq}}$
to denote the number of subvectors in
OCKM, ECKM, CKM, and PQ,
respectively.
The code length of CKM is $M_{\text{ck}}\log_2(K)$,
while the code length of OCKM is $M_{\text{ock}}C\log_2(K)$.
Fixing $C = 2$
as the analysis in Sec.~\ref{subsubsec:solve_b}, we set
$M_{\text{ock}} = M_{\text{ck}} / 2$
with $M_{\text{ck}}$ being $4$, $8$, and $16$ for code length $32$, $64$ and $128$, respectively.
The $M_{\text{pq}}$ is identical with $M_{\text{ck}}$,
while $M_{\text{eck}}$ is with $M_{\text{ock}}$.
In this way, the code length is identical through all
the approaches.

The results in terms of recall on SIFT1M, GIST1M, and SIFT1B
are shown in Fig.~\ref{fig:both_rec_32_64_128}.
From these results, we can see that:

\begin{figure}[t]
\centering
\begin{tabular}{c@{}c}
\includegraphics[width = \figSingleTwoWidth\linewidth]{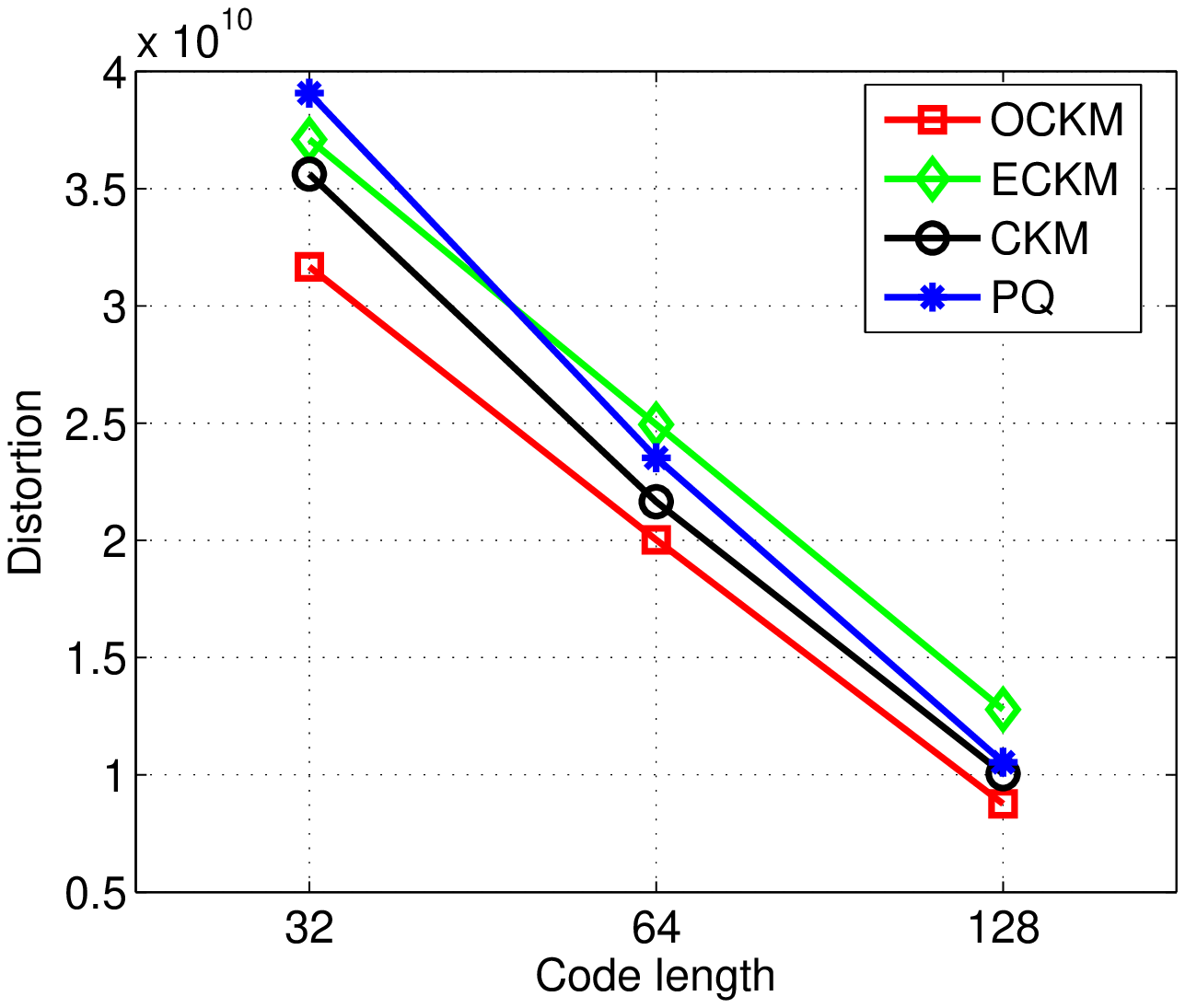}
&
\includegraphics[width = \figSingleTwoWidth\linewidth]{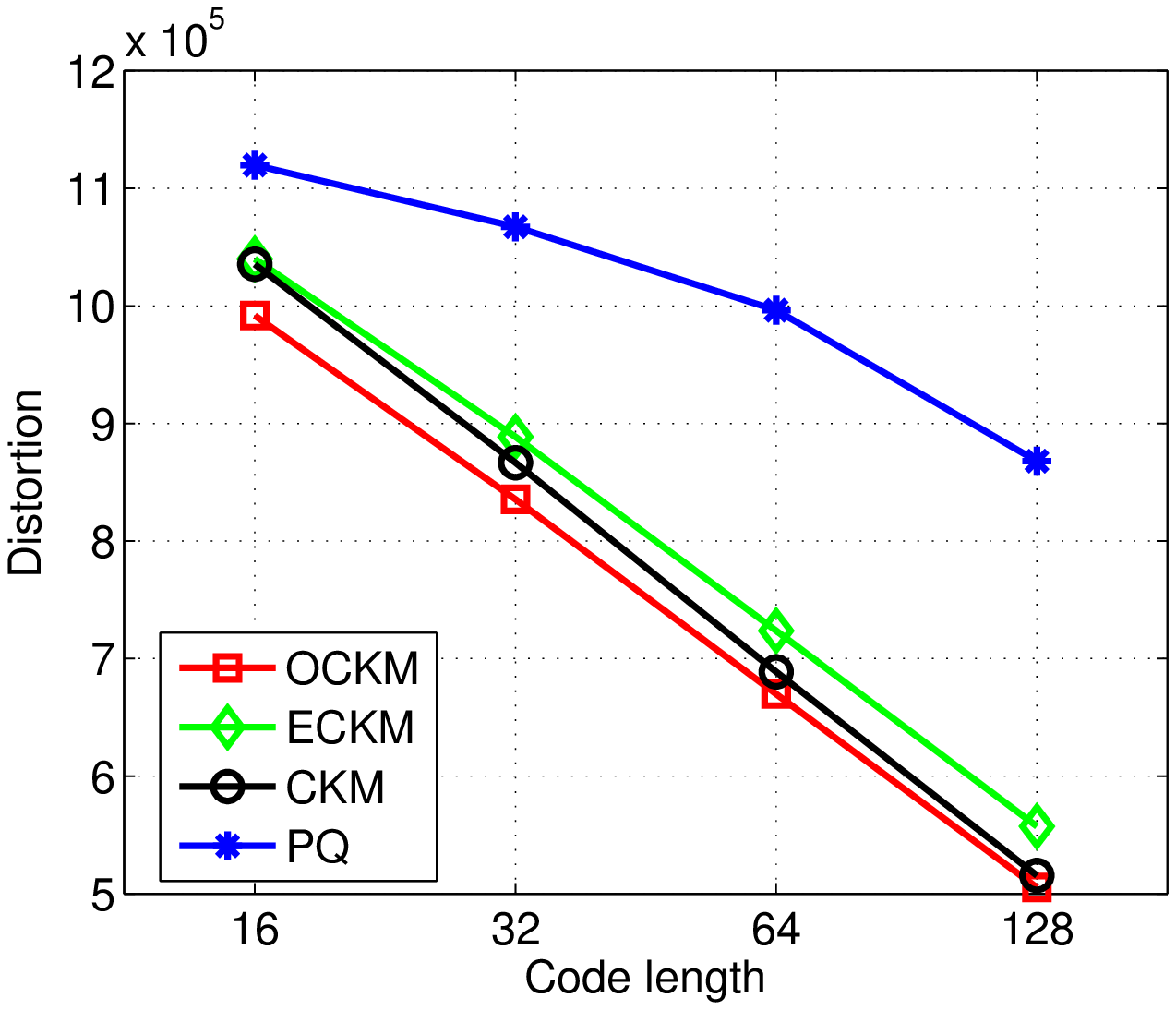}
\\
(a) SIFT1M & (b) GIST1M
\end{tabular}
\caption{Distortion under the same code length on the database set.}
\label{fig:distortion_code_length}
\end{figure}

\begin{figure*}
\begin{tabular}{c@{}c@{}c}
\includegraphics[width = \figDoubleThreeWidth\linewidth]{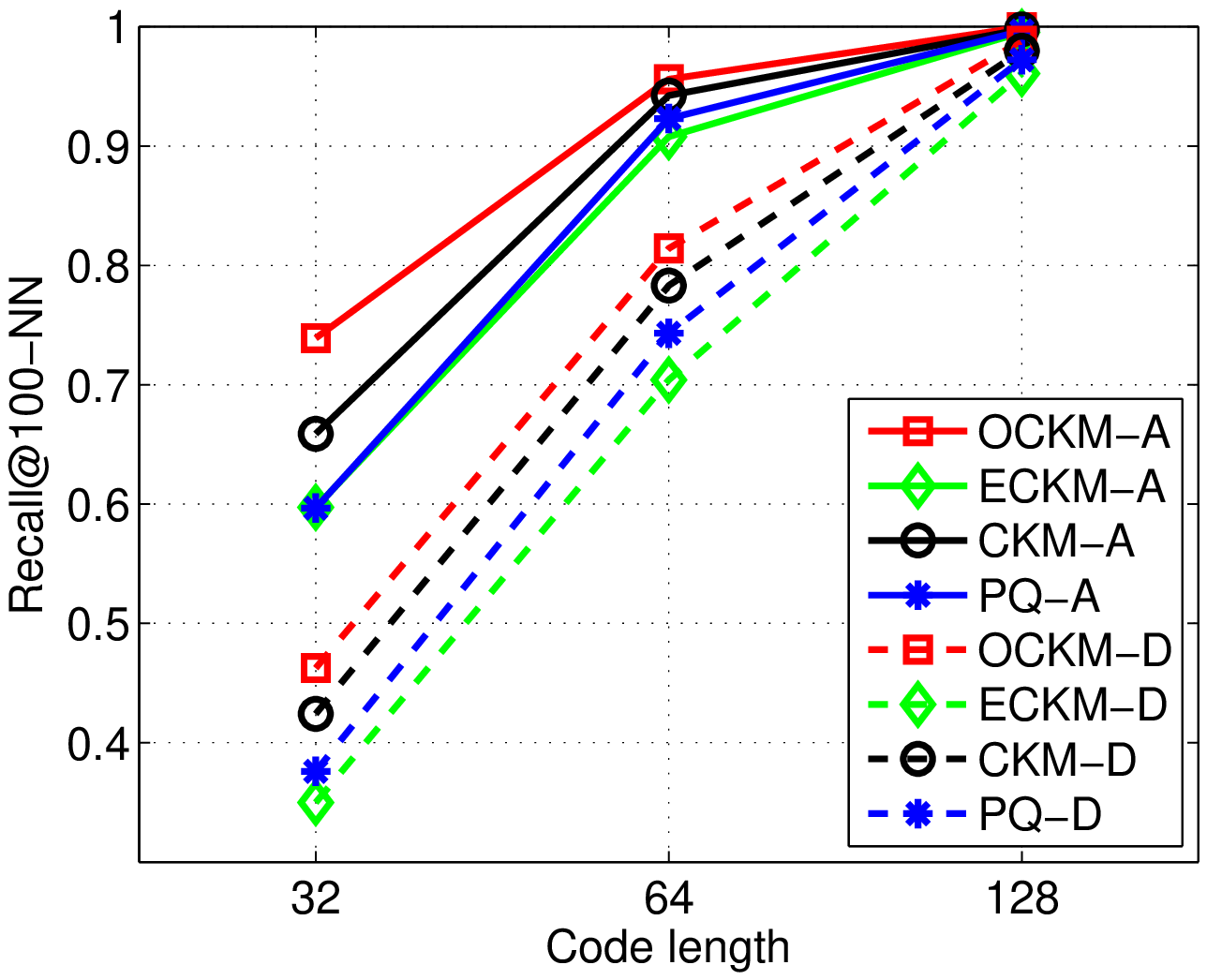}
&
\includegraphics[width = \figDoubleThreeWidth\linewidth]{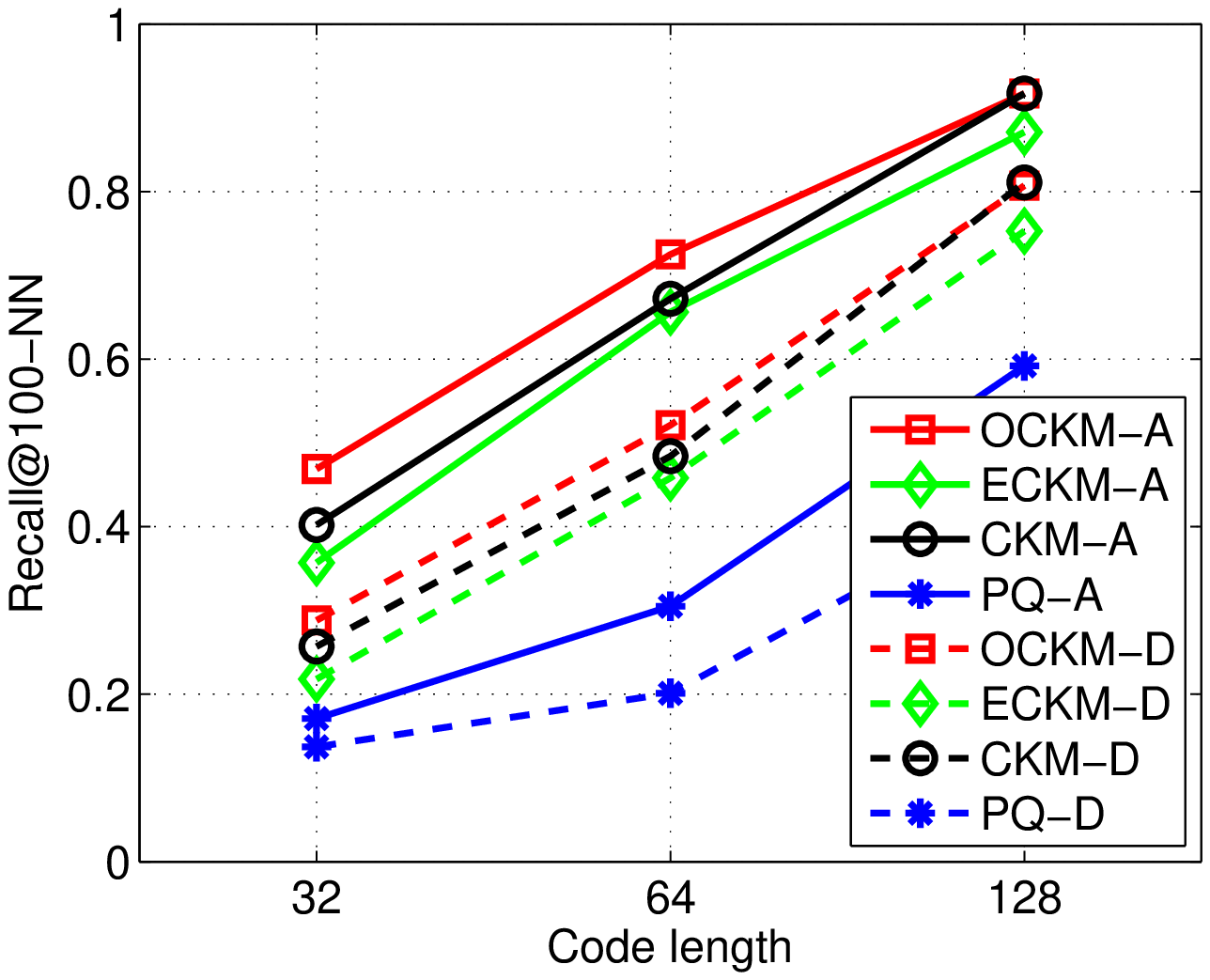}
&
\includegraphics[width = \figDoubleThreeWidth\linewidth]{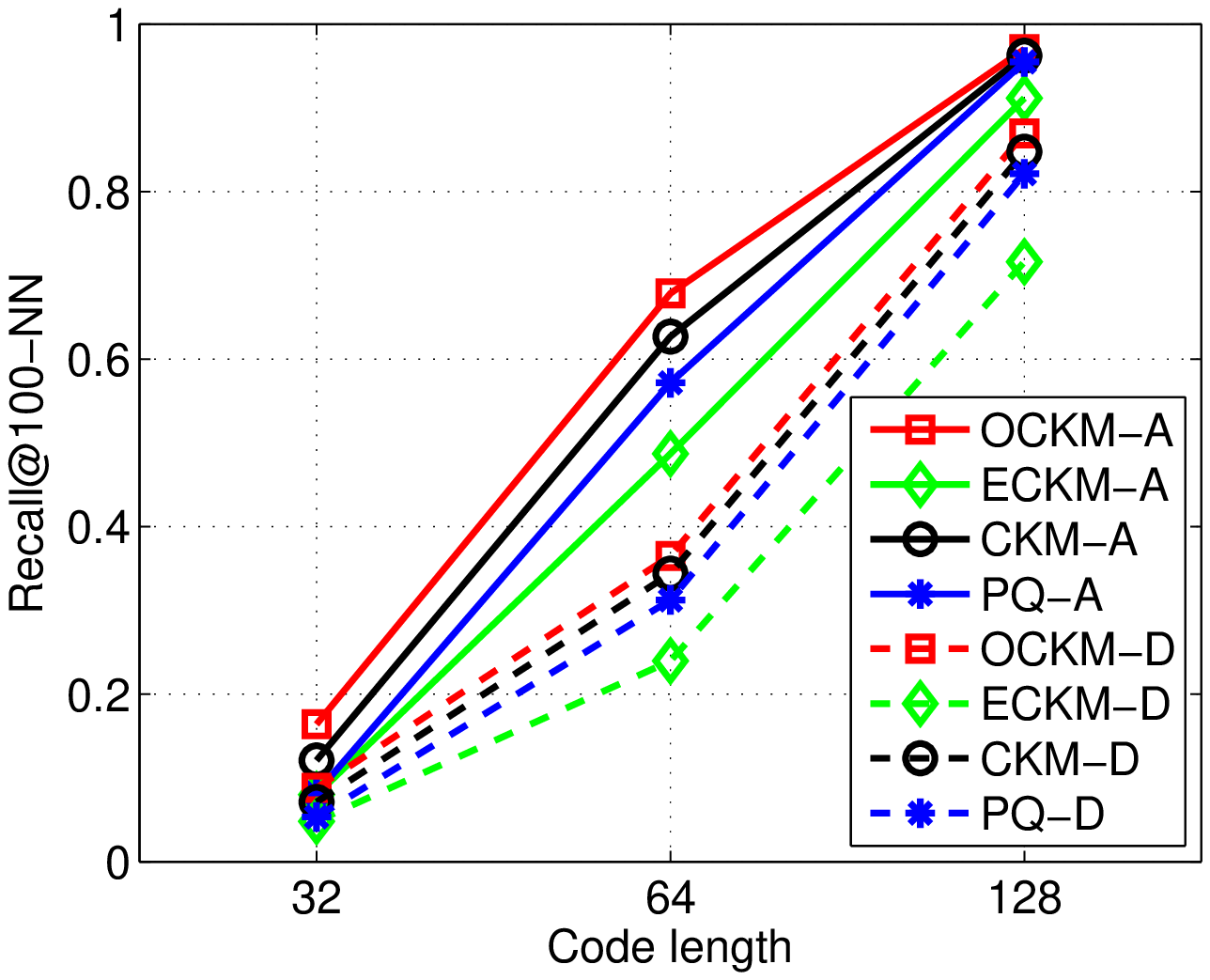} \\
(a) SIFT1M & (b) GIST1M & (c) SIFT1B
\end{tabular}
\caption{Recall at the $100$-th top ranked point under
the same code length.}
\label{fig:recall100_code}
\end{figure*}

\begin{itemize}
\item
Generally, our OCKM outperforms all the others under the same type of approximate distance.
For example of the asymmetric distance with 64 bits,
the improvement of OCKM is about $5$ percents on SIFT1M
in Fig.~\ref{fig:both_rec_32_64_128} (b), $4$ percents
on GIST1M in Fig.~\ref{fig:both_rec_32_64_128} (e),
$4$ percents on SIFT1B in Fig.~\ref{fig:both_rec_32_64_128} (h)
at the $10$-th top ranked point.
The performance of OCKM mainly benefits from
the low distortion errors,
which is also discussed in Theorem~\ref{thm:code_length}.
Fig.~\ref{fig:distortion_code_length}
illustrates the distortion on the database under
the same code length for SIFT1M and GIST1M.
We can see under the same code length, our approach
achieves the lowest distortions.

\item The improvement is even better with a smaller code length.
To present the observation more clearly, we
extract the
recall at the $100$-th nearest neighbor from Fig.~\ref{fig:both_rec_32_64_128}
and plot Fig.~\ref{fig:recall100_code}.
With a larger code length, the recalls of
our OCKM and the second best CKM approach $1$.
With a smaller code length, our OCKM
gains larger improvement.


\item
ECKM is not quite competitive with the same code length. The possible reason is that the number of sub codebooks is smaller than those of the others. Take the code length of $64$ bits as an example.
There are $8$ subvectors and each has one sub codebook for PQ and CKM, resulting in $8$ sub codebooks. OCKM is equipped with $4$ subvectors, but each has two sub codebooks, also resulting in $8$ sub codebooks.
Comparatively, ECKM has $4$ subvectors, each of which has one sub codebook, and there are only $4$ sub codebooks in total.
Smaller numbers of sub codebooks may degrade the
performance of ECKM.
Compared with SIFT1M and SIFT1B,
ECKM achieves even better results than PQ on GIST1M,
which indicates GIST1M is more sensitive to the rotation.


\end{itemize}

Fig.~\ref{fig:both_32_64_128}
illustrates the experiment results in terms of mean overall ratio with different code lengths on SIFT1M and GIST1M.
Mean overall ratio captures the whole quality of the returned points while the recall captures the position of the nearest neighbor and ignores the quality of the other points.
Under this criterion, our OCKM achieves the lowest mean overall ratio and outperforms all the others.
This implies the returned nearest neighbors of OCKM are
of high quality and close to the query points.

\begin{figure*}
\centering
\begin{tabular}{c@{}c@{}c}
32	& 	64	& 128 \\
\includegraphics[width = 0.32\linewidth]{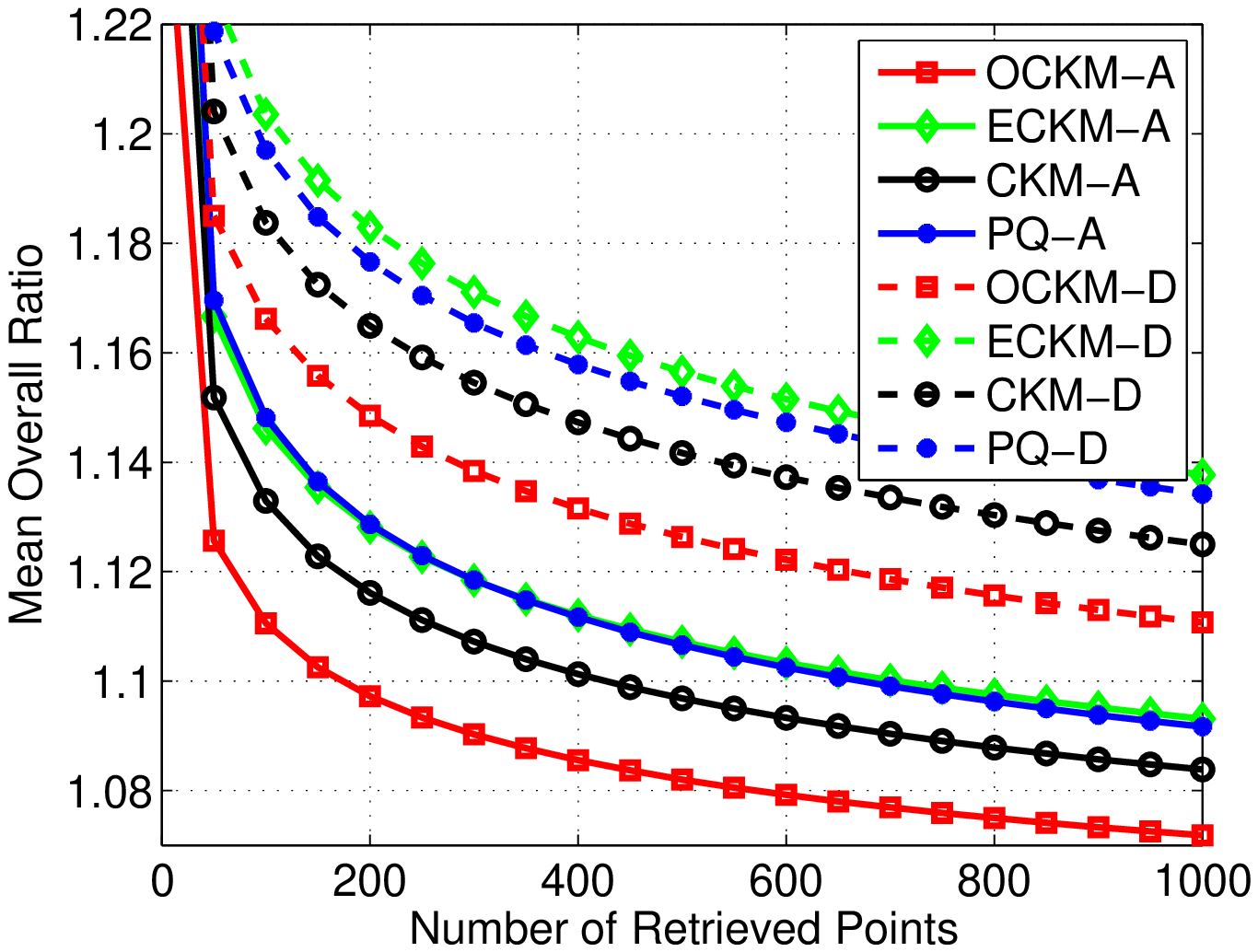} &
\includegraphics[width = 0.32\linewidth]{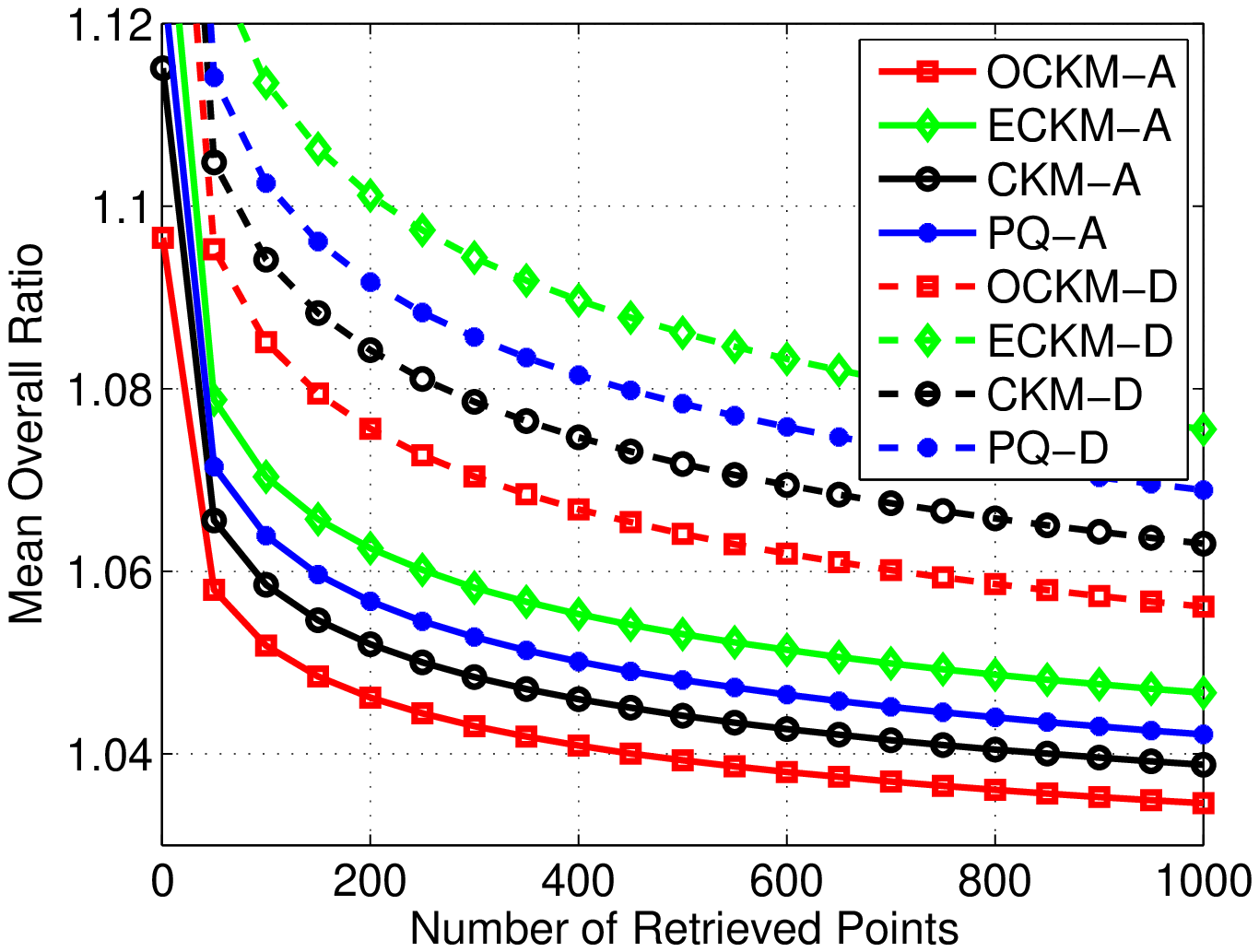} &
\includegraphics[width = 0.32\linewidth]{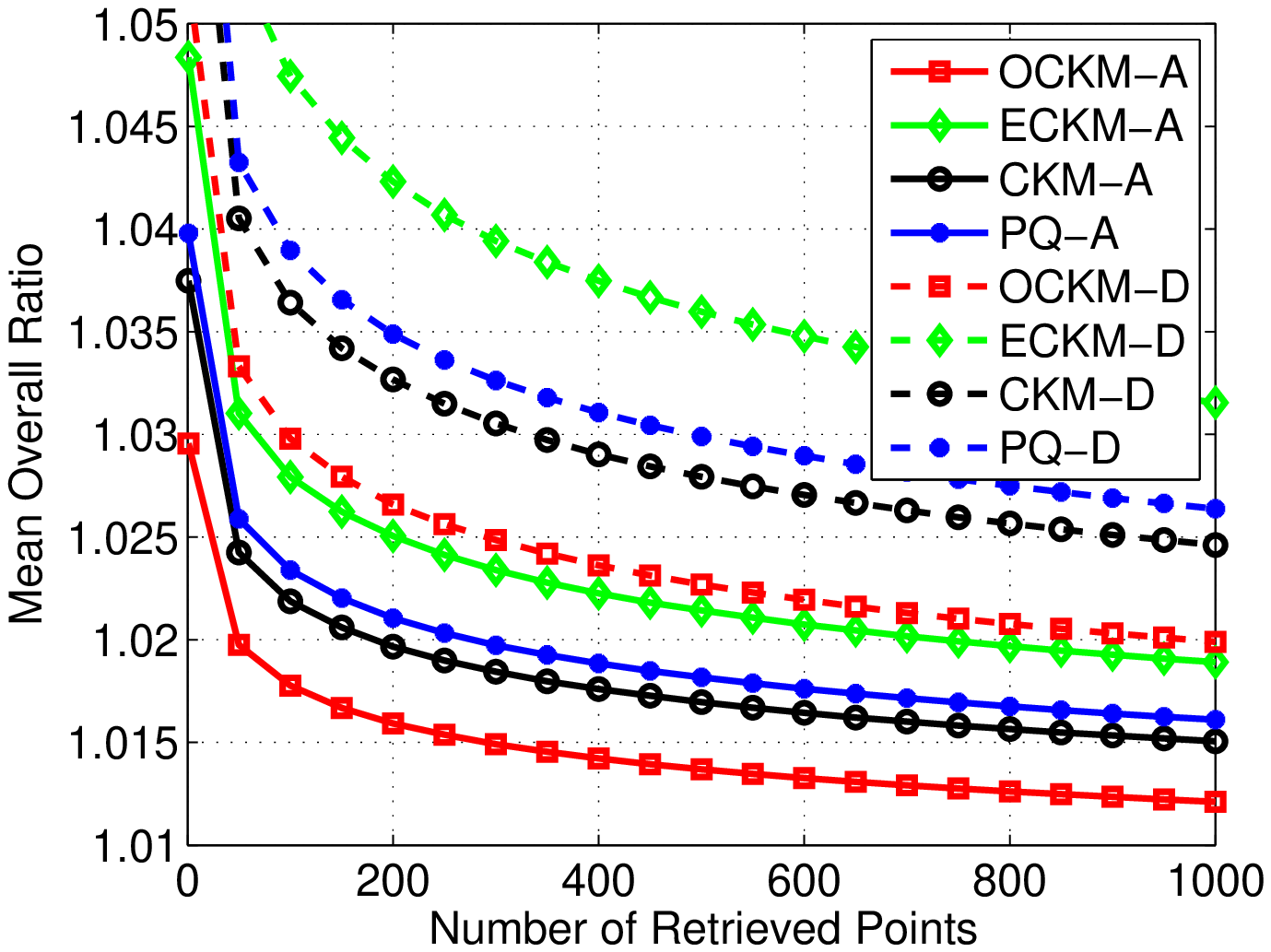} \\
\includegraphics[width = 0.32\linewidth]{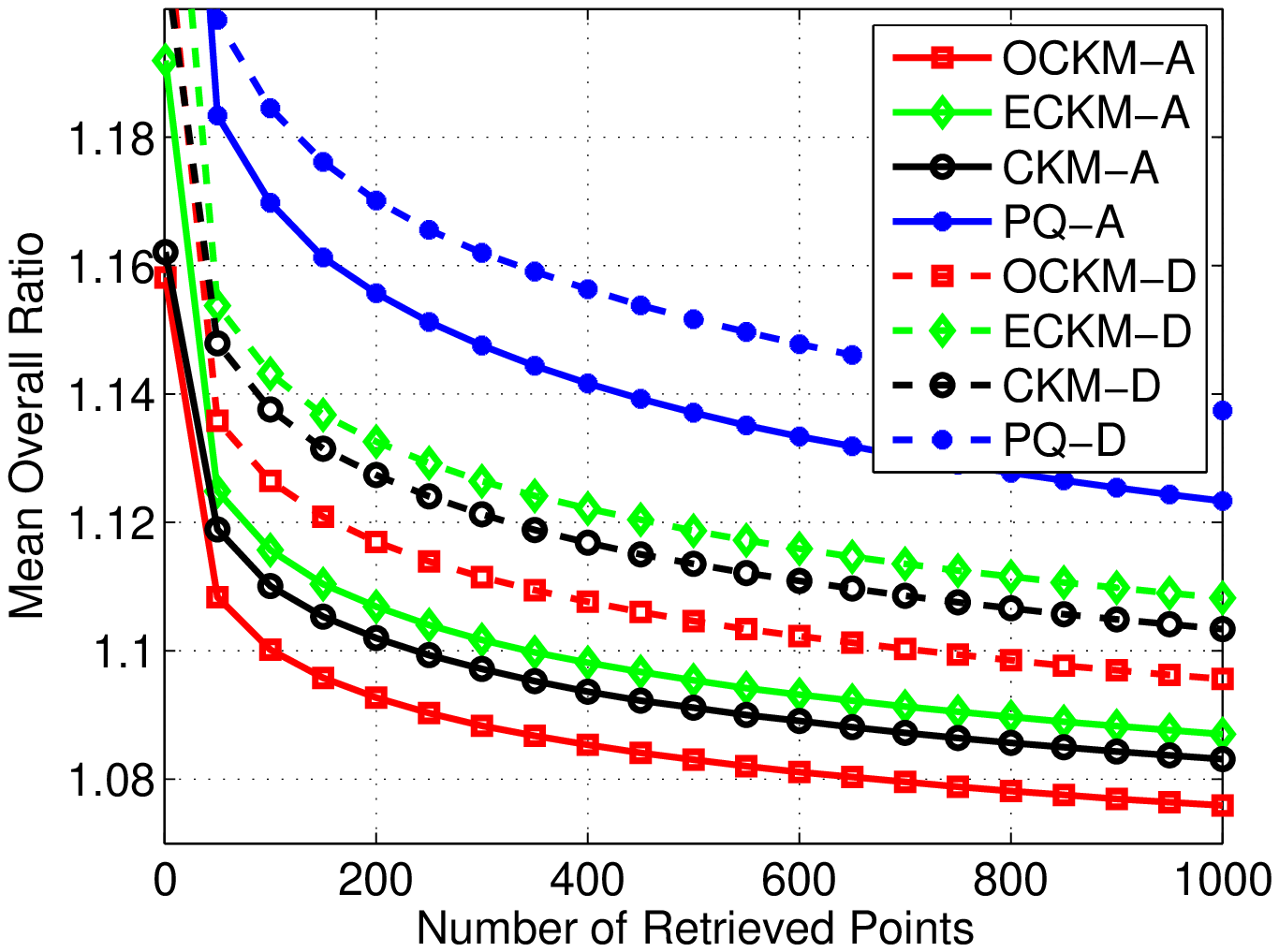} &
\includegraphics[width = 0.32\linewidth]{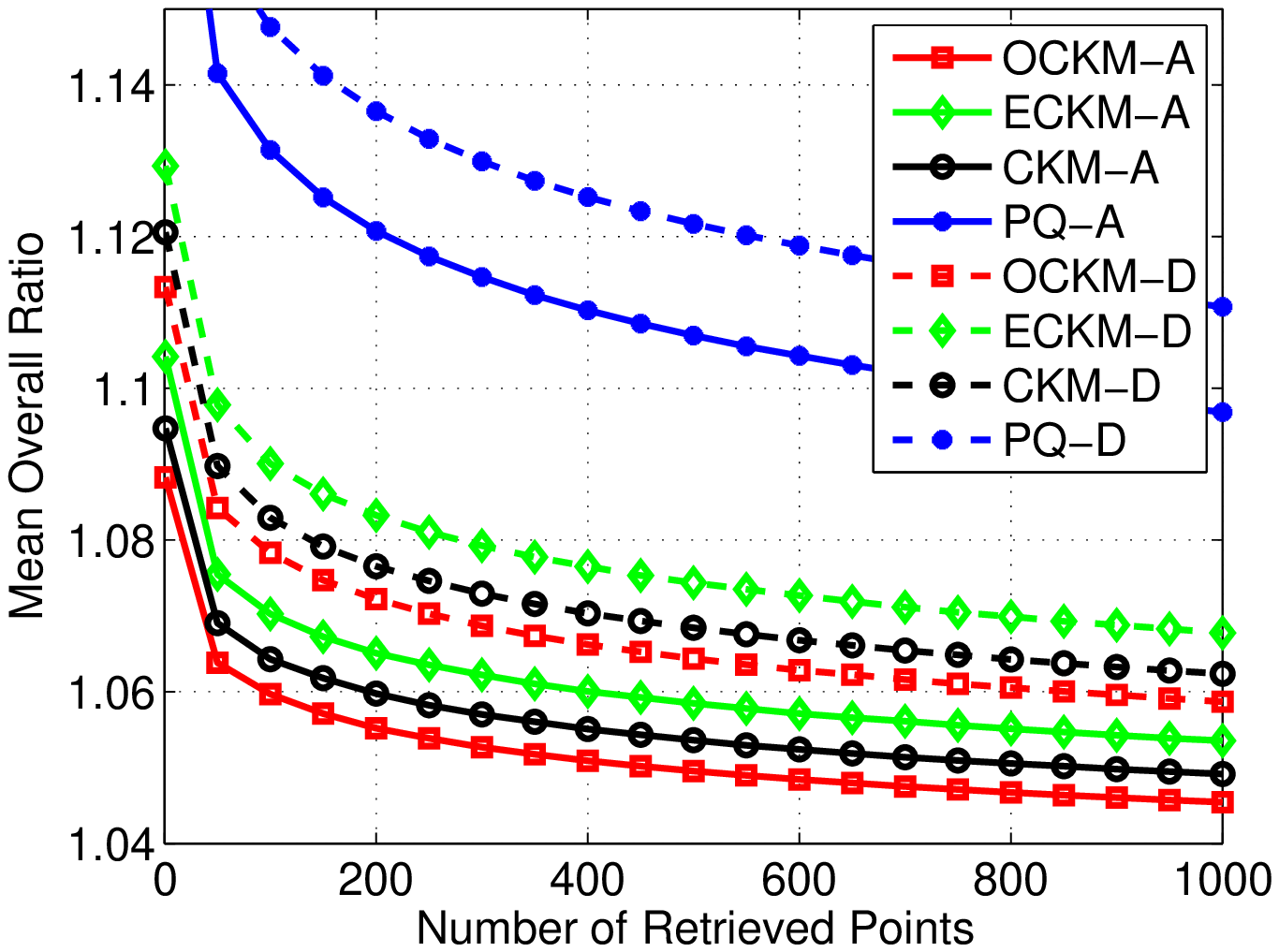} &
\includegraphics[width = 0.32\linewidth]{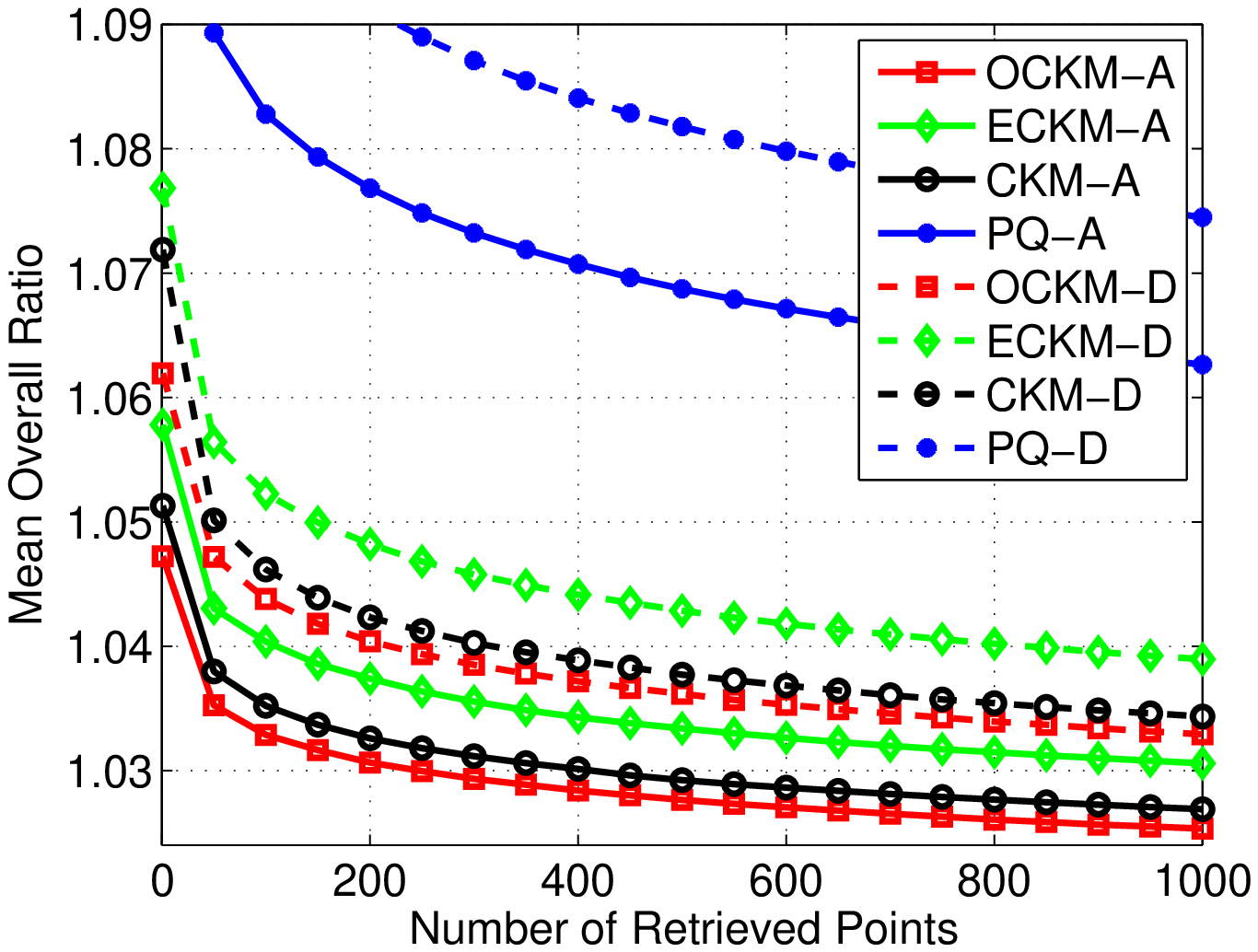} \\
\end{tabular}
\caption{Mean overall ratio for ANN search. The results in the first row are on SIFT1M while those in the second row are on GIST1M. The first column corresponds to the code length $32$; the second to $64$; and the third to $128$.}
\label{fig:both_32_64_128}
\end{figure*}

%
%
%

\section{Conclusion}
\label{sec:conclusion}
In this paper, we proposed the Optimized Cartesian $K$-Means (OCKM) algorithm
to encode the high-dimensional data points
for approximate nearest neighbor search.
The key idea of OCKM
is that in each subspace multiple sub codebooks are generated
and each sub codebook contributes one sub codeword 
for encoding the subvector.
The benefit is that it reduces the quantization error
with comparable query time
under the same code length.
The theoretical analysis
and experimental results
show that
OCKM achieves superior performance for ANN search over state-of-the-art approaches.

\section*{Acknowledgment}
This work was partially supported by the National Basic Research Program
of China (973 Program) under Grant 2014CB347600
and ARC Discovery Project DP130103252.

%
%
%

\ifCLASSOPTIONcaptionsoff
  \newpage
\fi

\bibliographystyle{abbrv}
\bibliography{ockmeans}

\begin{IEEEbiography}[{\includegraphics[width=1in,height=1.25in,clip,keepaspectratio]{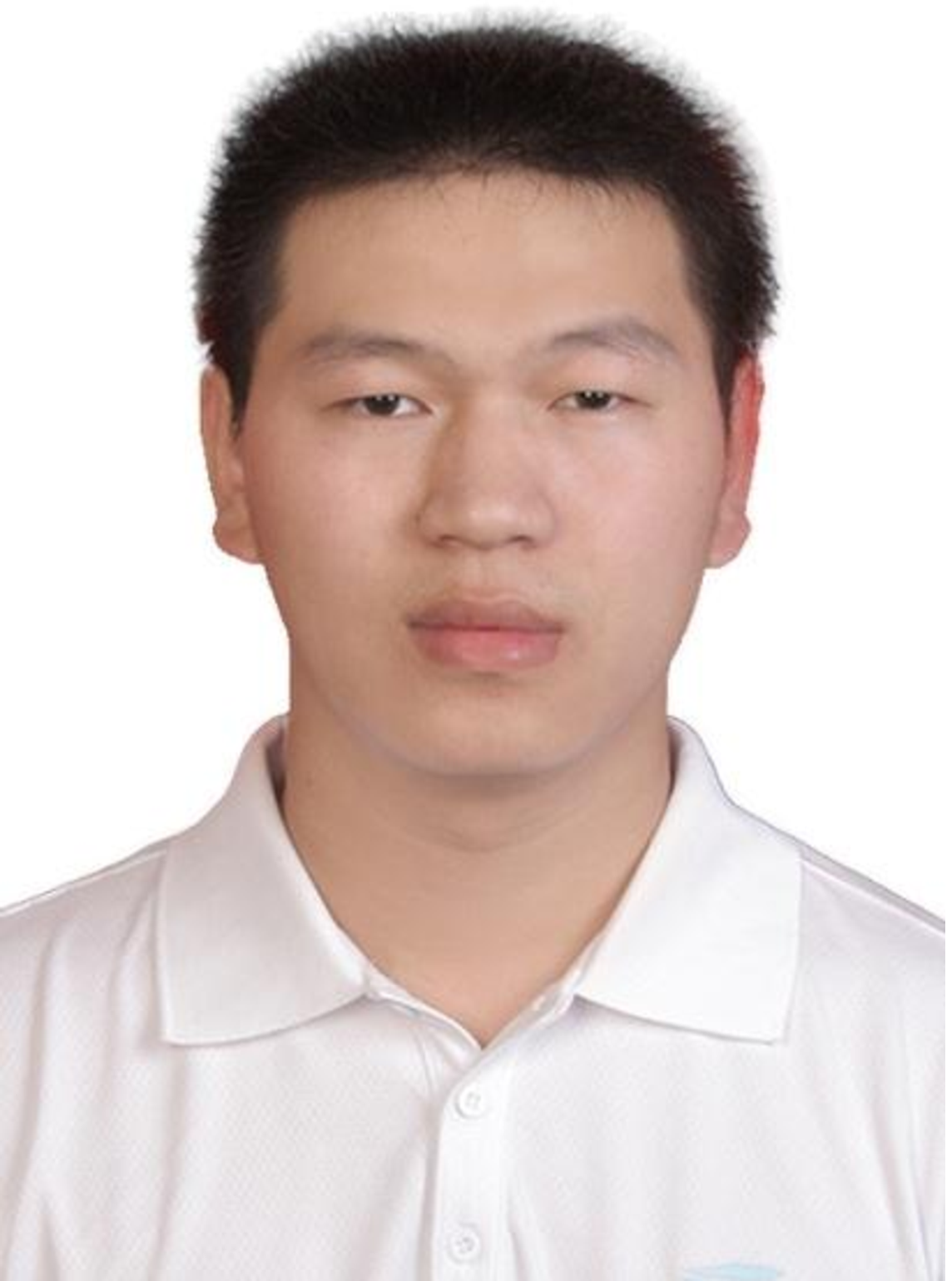}}]{Jianfeng Wang} received his
B.Eng. degree from the Department of Electronic Engineering
and Information Science in the University of Science and
Technology of China (USTC) in 2010. Currently, he is
a PhD student in MOE-Microsoft Key Laboratory of
Multimedia Computing and Communication, USTC. His
research interests include multimedia retrieval, machine learning
 and its applications.
\end{IEEEbiography}

\begin{IEEEbiography}[{\includegraphics[width=1in,height=1.25in,clip,keepaspectratio]{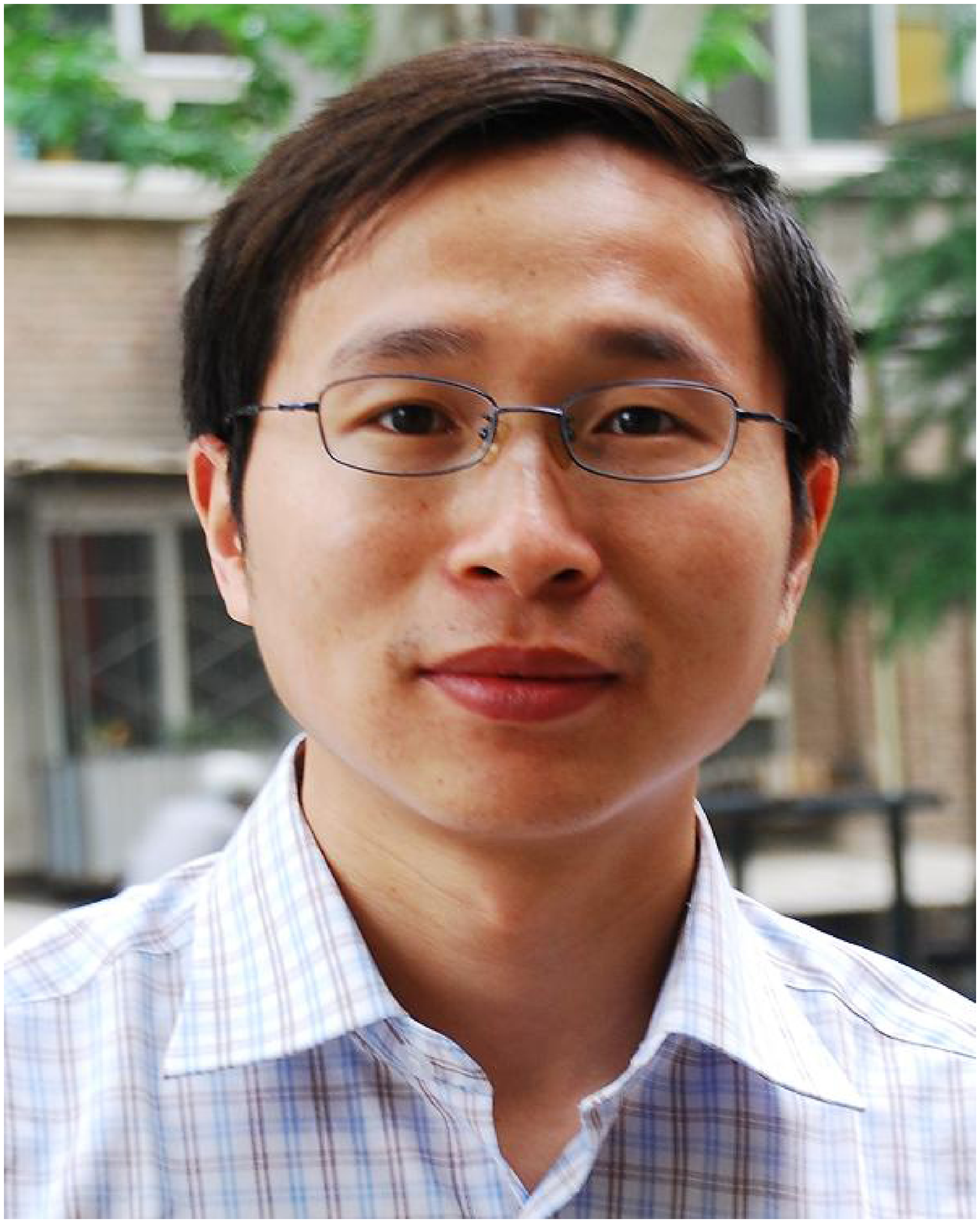}}]
{Jingdong Wang} received the BSc and MSc degrees in Automation from Tsinghua University,
Beijing, China, in 2001 and 2004, respectively,
and the PhD degree in Computer Science from the Hong Kong University of Science and Technology,
Hong Kong, in 2007.
He is currently a Lead Researcher at the Visual Computing Group,
Microsoft Research, Beijing, P.R. China.
His areas of interest include computer vision, machine learning, and multimedia search.
At present, he is mainly working on the Big Media project, including large-scale indexing and
clustering, and Web image search and mining. He is an editorial board member of Multimedia Tools and Applications.
\end{IEEEbiography}


\begin{IEEEbiography}[{\includegraphics[width=1in,height=1.25in,clip,keepaspectratio]{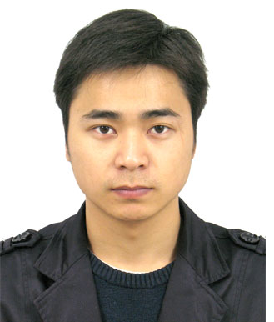}}]
{Jingkuan Song} is currently a Research Fellow in University of Trento, Italy. He received his Ph.D degree from The University of Queensland, and BS degree in Software Engineering from University of Electronic Science and Technology of China. His research interest includes large-scale multimedia search, computer vision and machine learning.
\end{IEEEbiography}

\begin{IEEEbiography}[{\includegraphics[width=1in,height=1.25in,clip,keepaspectratio]{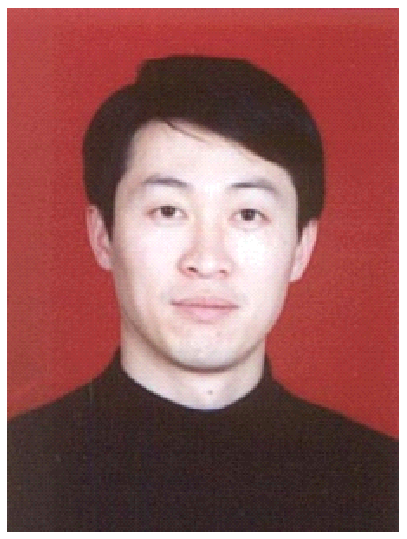}}] {Xin-Shun Xu} received his M.S. and Ph.D. Degrees in computer science from Shandong University, China, in 2002, and Toyama University, Japan, in 2005, respectively. He joined the School of Computer Science and Technology at Shandong University as an associate professor in 2005, and joined the LAMDA group of the National Key Laboratory for Novel Software Technology, Nanjing University, China, as a postdoctoral fellow in 2009. Currently, he is a professor of the School of Computer Science and Technology at Shandong University, and the leader of MIMA (Machine Intelligence and Media Analysis) group of Shandong University. His research interests include machine learning, information retrieval, data mining, bioinformatics, and image/video analysis.
\end{IEEEbiography}

\begin{IEEEbiography}[{\includegraphics[width=1in,height=1.25in,clip,keepaspectratio]{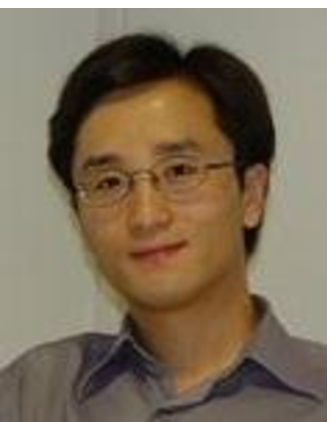}}]{Heng Tao Shen} is a Professor of Computer Science in School of Information Technology and Electrical
Engineering, The University of Queensland. He obtained his B.Sc. (with 1st class Honours) and Ph.D.
from Department of Computer Science, National University of Singapore in 2000 and 2004 respectively.
He then joined the University of Queensland as a Lecturer and became a Professor in 2011. His
research interests include Multimedia/Mobile/Web Search and Big Data Management. He is the winner
of Chris Wallace Award for outstanding Research Contribution in 2010 from CORE Australasia. He is an Associate Editor of IEEE TKDE, and will serve as a PC Co-Chair for ACM Multimedia 2015.
\end{IEEEbiography}

\begin{IEEEbiography}[{\includegraphics[width=1in,height=1.25in,clip,keepaspectratio]{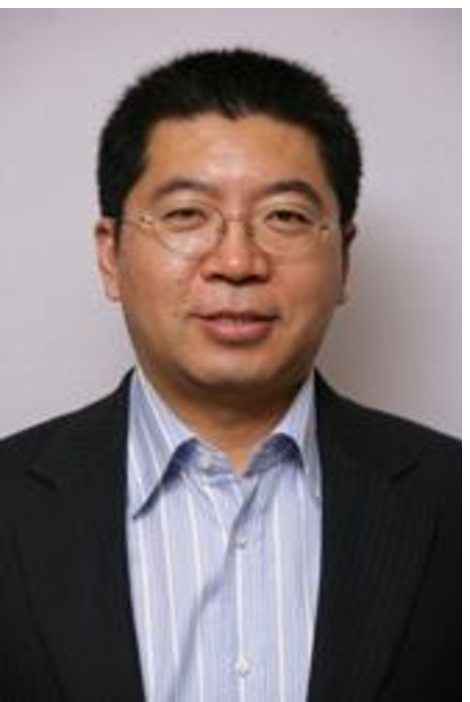}}]
{Shipeng Li} joined and helped to found Microsoft Research's Beijing lab in May 1999. He is now a Principal Researcher and Research Area Manager coordinating multimedia research activities in the lab. His research interests include multimedia processing, analysis, coding, streaming, networking and communications. From Oct. 1996 to May 1999, Dr. Li was with Multimedia Technology Laboratory at Sarnoff Corporation as a Member of Technical Staff. Dr. Li has been actively involved in research and development in broad multimedia areas and international standards. He has authored and co-authored 6 books/book chapters and 280+ referred journal and conference papers. He holds 140+ granted US patents.
Dr. Li received his B.S. and M.S. in Electrical Engineering (EE) from the University of Science and Technology of China (USTC), Hefei, China in 1988 and 1991, respectively. He received his Ph.D. in EE from Lehigh University, Bethlehem, PA, USA in 1996. He was a faculty member in Department of Electronic Engineering and Information Science at USTC in 1991-1992.
Dr. Li received the Best Paper Award in IEEE Transaction on Circuits and Systems for Video Technology (2009). Dr. Li is a Fellow of IEEE.

\end{IEEEbiography}

\end{document}